\pdfoutput=1
\documentclass[11pt]{article}
\usepackage[final]{acl}
\usepackage{times}
\usepackage{latexsym}
\usepackage[T1]{fontenc}
\usepackage[utf8]{inputenc}
\usepackage{microtype}
\usepackage{inconsolata}
\usepackage{booktabs}
\usepackage{colortbl}
\usepackage{stfloats}
\usepackage{multirow}
\usepackage{amsmath,amssymb}
\usepackage{amsthm}
\usepackage{stmaryrd}
\newcommand{\sem}[1]{\llbracket #1 \rrbracket}
\usepackage{enumitem}
\newtheorem{theorem}{Theorem}[section]

\usepackage{tikz}
\usepackage{graphicx}
\usetikzlibrary{arrows.meta,shadings}

\title{Principles of Concept Representation in Sentence Encoders}

\author{
 \textbf{Isabelle Mohr\textsuperscript{1,2}},
 \textbf{John Dujany\textsuperscript{2}},
 \textbf{Jonathan Souquet\textsuperscript{2}},
 \textbf{Andre Freitas\textsuperscript{1}},
\\
 \textsuperscript{1}Idiap Research Institute,
 \textsuperscript{2}Merck KGaA,
\\
 \small{
   \textbf{Correspondence:} \href{mailto:isabelle.mohr@idiap.ch}{isabelle.mohr@idiap.ch}
 }
}
\begin{document}
\maketitle

\begin{abstract}
What makes a sentence encoder produce good concept representations?
We approach this through the lens of representational compositionality: an encoder supports a concept family only when its latent space admits a low-distortion realization of the corresponding semantic operator.
This framing predicts both where current encoders succeed and where they are structurally mismatched to their supervision.
Through a controlled ablation over encoder conditions trained on 3.3 million synonym and definition pairs from WordNet and Wiktionary, evaluated on three decontaminated splits and a modifier-labeled noun-phrase benchmark, we identify four principles.
Fine-tuning recalibrates the latent geometry rather than expanding it~(P1).
Semantic signal concentrates in the final transformer layer before concept-specific training begins, making cross-layer pooling redundant~(P2).
Hard negatives improve discrimination and stress-test robustness without improving retrieval ranking, showing that calibration and ranking are independently addressable~(P3).
Finally, the effectiveness of supervision depends on the composition type of the target concept. Extensional training helps intersective and subsective families while degrading relational and intensional ones, exposing a structural limitation of current training paradigms~(P4).
We release two new evaluation datasets: a DBpedia semantic-gap benchmark and a modifier-labeled NP paraphrase suite.
\end{abstract}

\colorlet{tC}{teal!80!black}      
\colorlet{nC}{orange!80!red}      
\colorlet{rC}{blue!65!black}      
\colorlet{xC}{red!58!black}       

\begin{figure*}[t]
\centering
\resizebox{\linewidth}{!}{%
\begin{tikzpicture}[font=\small, >=Latex, yscale=0.82]

\begin{scope}

\draw[gray!28, fill=gray!4, rounded corners=10pt]
  (-0.3,-0.3) rectangle (6.2,7.2);
\node[font=\small\bfseries] at (2.95,6.88) {Frozen Encoder (B0)};
\node[font=\scriptsize, gray!65] at (2.95,6.47) {anisotropy 0.126};

\fill[gray!10] (2.95,3.35) ellipse (2.6 and 2.5);
\draw[gray!25, line width=0.5pt] (2.95,3.35) ellipse (2.6 and 2.5);

\fill[rC] (0.72,4.82) circle (3pt);
\node[rC, font=\scriptsize, align=right, anchor=east] at (0.58,4.82)
  {\textit{``where someone}\\\textit{drew last breath''}};

\fill[rC] (5.52,1.98) circle (3pt);
\node[rC, font=\scriptsize, anchor=west] at (5.65,1.98) {\texttt{deathPlace}};

\draw[<->, dashed, rC!48, line width=0.85pt]
  (1.08,4.62) to[bend right=8] (5.32,2.08);
\node[rC!60, font=\scriptsize] at (2.95,3.52) {\textit{semantic gap}};

\fill[tC] (4.88,5.48) circle (3pt);
\node[tC, font=\scriptsize, anchor=south] at (4.88,5.60) {\textit{dirty cup}};

\fill[tC] (5.55,5.02) circle (3pt);
\node[tC, font=\scriptsize, anchor=west] at (5.68,5.02) {\textit{filthy cup}};

\fill[tC!62] (0.72,5.92) circle (3pt);
\node[tC!72, font=\scriptsize, anchor=south] at (0.72,6.04)
  {\textit{not clean cup}};

\draw[<->, dashed, tC!42, line width=0.7pt]
  (1.06,5.89) to[bend left=15] (4.59,5.52);
\node[tC!55, font=\tiny] at (2.88,6.13) {\textit{surface form splits}};

\fill[nC] (2.28,1.28) circle (3pt);
\node[nC, font=\scriptsize, anchor=north] at (2.08,1.15) {\textit{a mammal}};

\fill[nC!62] (3.08,1.38) circle (3pt);
\node[nC!72, font=\scriptsize, anchor=north] at (3.28,1.25)
  {\textit{not a mammal}};

\draw[nC!38, line width=0.6pt] (2.45,1.32) -- (2.91,1.36);
\node[nC!65, font=\tiny] at (2.68,1.52) {\textit{indistinguishable}};

\end{scope}

\draw[->, line width=1.8pt, gray!38] (6.42,3.35) -- (7.73,3.35);

\begin{scope}[xshift=8.0cm]

\draw[gray!28, fill=gray!4, rounded corners=10pt]
  (-0.3,-0.3) rectangle (6.2,7.2);
\node[font=\small\bfseries] at (2.95,6.88)
  {Concept-Equivalence Supervision};
\node[font=\scriptsize, gray!65] at (2.95,6.45)
  {InfoNCE + 0.5\,BCE $\cdot$ 3.3M pairs};

\draw[gray!18, line width=0.4pt] (0.08,4.32) -- (6.08,4.32);
\draw[gray!18, line width=0.4pt] (0.08,2.50) -- (6.08,2.50);

\node[font=\tiny\bfseries, tC!82, anchor=east] at (0.06,5.52) {syn};

\node[draw=tC!48, fill=tC!10, rounded corners=3pt,
      font=\scriptsize, minimum width=1.55cm, align=center, inner sep=3.5pt]
  (synL) at (1.22,5.52) {\textit{dirty cup}};

\node[draw=tC!48, fill=tC!10, rounded corners=3pt,
      font=\scriptsize, minimum width=1.6cm, align=center, inner sep=3.5pt]
  (synR) at (4.82,5.52) {\textit{not clean cup}};

\draw[<->, tC!68, line width=1.25pt] (synL.east) -- (synR.west);
\node[font=\tiny, tC!85, anchor=south] at (3.02,5.60) {$\uparrow$\,sim};

\node[font=\tiny\bfseries, rC!82, anchor=east] at (0.06,3.42) {t2d};

\node[draw=rC!42, fill=rC!8, rounded corners=3pt,
      font=\scriptsize, minimum width=1.55cm, align=center, inner sep=3.5pt]
  (t2dL) at (1.22,3.42) {\texttt{deathPlace}};

\node[draw=rC!42, fill=rC!8, rounded corners=3pt,
      font=\scriptsize, text width=2.0cm, align=center, inner sep=3.5pt]
  (t2dR) at (4.88,3.42)
  {\textit{``where someone\\drew last breath''}};

\draw[<->, rC!62, line width=1.25pt] (t2dL.east) -- (t2dR.west);
\node[font=\tiny, rC!82, anchor=south] at (3.02,3.50) {$\uparrow$\,sim};

\node[font=\tiny\bfseries, nC!82, anchor=east] at (0.06,1.50) {hn};

\node[draw=nC!42, fill=nC!8, rounded corners=3pt,
      font=\scriptsize, minimum width=1.55cm, align=center, inner sep=3.5pt]
  (hnL) at (1.22,1.50) {\textit{a mammal}};

\node[draw=nC!42, fill=nC!8, rounded corners=3pt,
      font=\scriptsize, minimum width=1.6cm, align=center, inner sep=3.5pt]
  (hnR) at (4.82,1.50) {\textit{not a mammal}};

\draw[<->, xC!65, dashed, line width=1.1pt] (hnL.east) -- (hnR.west);
\node[font=\tiny, xC!82, anchor=south] at (3.02,1.58) {$\downarrow$\,sim};

\draw[gray!55, <->, line width=1.1pt] (0.15,0.25) -- (0.65,0.25);
\node[font=\tiny, gray!68, anchor=west] at (0.72,0.25)
  {pull together (equiv.\ pairs)};

\draw[xC!58, <->, dashed, line width=1.0pt] (0.15,0.03) -- (0.65,0.03);
\node[font=\tiny, gray!68, anchor=west] at (0.72,0.03)
  {push apart (hard negatives)};

\end{scope}

\draw[->, line width=1.8pt, gray!38] (14.42,3.35) -- (15.80,3.35);
\node[font=\scriptsize, gray!60, align=center] at (15.11,4.05)
  {reshapes\\geometry};

\begin{scope}[xshift=16.1cm]

\draw[gray!28, fill=gray!4, rounded corners=10pt]
  (-0.3,-0.3) rectangle (6.2,7.2);
\node[font=\small\bfseries] at (2.95,6.88) {Fine-tuned Encoder (B1)};
\node[font=\scriptsize, gray!65] at (2.95,6.47) {anisotropy 0.012};

\fill[gray!10] (2.95,3.35) ellipse (2.45 and 2.5);
\draw[gray!25, line width=0.5pt] (2.95,3.35) ellipse (2.45 and 2.5);

\fill[rC] (1.88,1.88) circle (3pt);
\node[rC, font=\scriptsize, align=center, anchor=north] at (1.88,1.72)
  {\textit{``where someone}\\\textit{drew last breath''}};

\fill[rC] (3.92,1.88) circle (3pt);
\node[rC, font=\scriptsize, anchor=north] at (3.92,1.72)
  {\texttt{deathPlace}};

\draw[->, rC!78, line width=1.25pt] (2.20,1.88) -- (3.60,1.88);
\node[rC!78, font=\tiny, anchor=south] at (2.90,1.95) {\textit{bridged}};

\fill[tC!12] (3.48,4.85) ellipse (1.45 and 0.88);
\draw[tC!38, dashed, line width=0.62pt] (3.48,4.85) ellipse (1.45 and 0.88);

\fill[tC] (2.48,5.02) circle (3pt);
\node[tC, font=\scriptsize, anchor=east] at (2.35,5.02) {\textit{not clean cup}};

\fill[tC] (3.58,5.42) circle (3pt);
\node[tC, font=\scriptsize, anchor=south] at (3.58,5.54) {\textit{dirty cup}};

\fill[tC] (4.38,4.85) circle (3pt);
\node[tC, font=\scriptsize, anchor=west] at (4.51,4.85) {\textit{filthy cup}};

\node[tC!78, font=\tiny] at (3.48,3.92) {\textit{concept cluster (P1)}};

\fill[nC] (5.32,5.78) circle (3pt);
\node[nC, font=\scriptsize, anchor=west] at (5.45,5.78) {\textit{a mammal}};

\fill[nC!62] (0.68,1.22) circle (3pt);
\node[nC!72, font=\scriptsize, anchor=east] at (0.55,1.22)
  {\textit{not a mammal}};

\node[nC!68, font=\tiny] at (2.95,3.52) {\textit{negation separated (P3)}};

\end{scope}

\end{tikzpicture}%
}
\caption{%
  Concept-equivalence fine-tuning, illustrated.
  \textbf{B0:} a frozen encoder splits synonymous paraphrases (\textcolor{teal!80!black}{teal}) by surface form, conflates negations (\textcolor{orange!80!red}{orange}), and leaves NL queries far from structured targets (\textcolor{blue!65!black}{blue}).
  \textbf{Centre:} synonym (\textit{syn}) and term-definition (\textit{t2d}) pairs are pulled together via InfoNCE; hard negatives (\textit{hn}) are pushed apart via BCE.
  \textbf{B1:} the space is recalibrated---paraphrases cluster (P1), the semantic gap bridges, and negation separates (P3).%
}
\label{fig:concept_space}
\end{figure*}
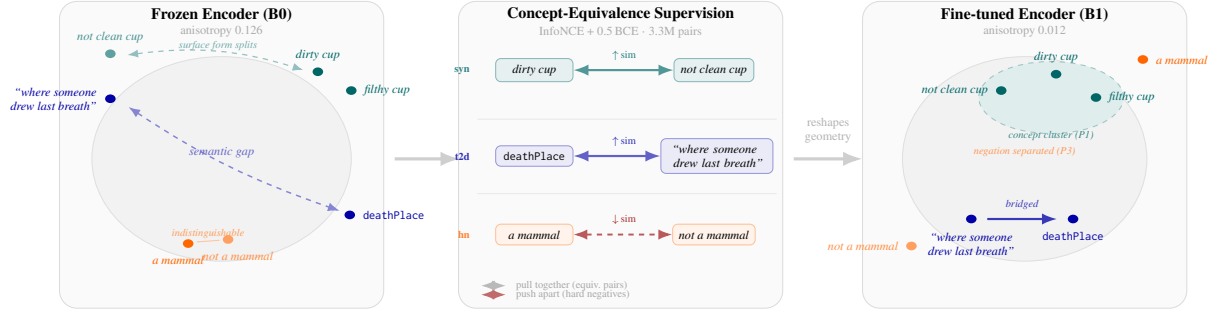

\section{Introduction}


Semantic compositionality is the principle in linguistics and philosophy that the meaning of a complex expression arises from the meanings of its individual components together with the way those components are structured and combined\citep{frege1892}.
In order for encoders to produce faithful concept representations, they must capture conceptual compositionality. Different modifier types such as intersective, subsective, relational, modal, and privative, contribute meaning through fundamentally different semantic operators \citep{carvalho2025montague}, yet a sentence encoder must realize all of them inside a single latent geometry scored by one similarity function.
Our theoretical lens is that conceptual compositionality in encoders is an approximate homomorphism problem (formalised in Appendix~\ref{app:theory}).
An encoder yields good concept representations only when the typed semantic operators required by a concept family admit low-distortion geometric realization.
The empirical question this paper asks and answers is which parts of that structure current sentence encoders already support, and which parts remain mismatched to the supervision used to train them.

We study this through concept retrieval, which makes representational quality of concept compositions directly measurable. The encoder must map a query to the same region as any denotationally equivalent target, regardless of surface form.
Complex nominals pose a hard challenge. Intersective, relational, and privative modifiers each realize meaning through different operators, yet all must coexist in a single latent geometry scored by one similarity function. This is a fundamental compositional tension that is illustrated geometrically in Figure~\ref{fig:concept_space}.
Evidence that this tension remains an open problem comes from our negation stress test, where the best frozen baseline scores 0.470 ROC-AUC (below chance), assigning higher similarity to negated definitions than to correct ones.

Several design dimensions are natural candidates for closing this gap.
Fine-tuning on concept-equivalence pairs may reshape the latent space toward the right geometry.
At the readout level, probing studies establish that upper transformer layers encode more semantic information than lower ones \citep{peters2018elmo,jawahar2019bert,tenney2019bert}, motivating cross-layer pooling as a candidate improvement.
Hard negative supervision may further sharpen discrimination between near-miss distractors.
And theoretically, there is no formal account of when a sentence encoder can support a given modifier composition family, nor of the geometric conditions under which current training objectives succeed or fail.

We structure the empirical investigation around three hypotheses.
\textbf{H1} \emph{(fine-tuning is necessary):} concept-equivalence fine-tuning substantially improves complex nominal retrieval over frozen baselines.
\textbf{H2} \emph{(cross-layer pooling):} weighted or input-adaptive mixtures over multiple layers improve over fine-tuned mean pooling.
\textbf{H3} \emph{(hard negatives):} hard negative supervision improves retrieval ranking in addition to calibration.

\paragraph{Contributions.}
\textbf{(T) $\varepsilon_\tau$-compositionality framework:} We introduce a formal characterisation of when a sentence encoder supports a given modifier composition family: $f_\theta$ is $\varepsilon_\tau$-compositional if a low-distortion latent operator $\Phi_\tau$ exists for modifier type $\tau$. This identifies two interacting bottlenecks (representational and objective) and predicts the modifier-family pattern (P4) from first principles. Formal bounds are derived in Appendix~\ref{app:theory}.
We additionally identify four empirical principles of compositional concept representation.
\textbf{(P1) Recalibration, not expansion:} concept-equivalence fine-tuning reshapes the latent geometry, collapsing anisotropy from 0.126 to 0.012 and improving term-to-definition Recall@10 from 0.552 to 0.654, while leaving effective rank unchanged. Fine-tuning recalibrates which regions of the space collapse together, without growing the space.
\textbf{(P2) Final-layer concentration precedes fine-tuning:} sentence-level pre-training already concentrates semantic signal into the final transformer layer, explaining why cross-layer pooling offers no consistent benefit after concept-equivalence fine-tuning. Confirmed across both NP-paraphrase and term-to-definition tasks. 
\textbf{(P3) Calibration and ranking are dissociable:} hard negatives improve discrimination (ROC-AUC $+$0.19--$+$0.46) without improving retrieval ranking, establishing that calibration and ranking are separable training targets.
\textbf{(P4) Supervision must match composition type:} concept-equivalence training improves intersective and subsective families while degrading relational and intensional types, exposing a fundamental mismatch between equivalence-only supervision and typed semantic operators.
We also release two new evaluation datasets: a DBpedia semantic-gap benchmark (3k train / 250 test, zero lexical overlap) and a modifier-labeled NP paraphrase suite (4,000 pairs across five composition families).


\section{Related Work}

\paragraph{Dense retrieval and contrastive learning.}
Sentence-BERT \citep{reimers2019sentencebert} establishes fine-tuned bi-encoders as practical dense retrievers, and SimCSE \citep{gao2021simcse} shows that in-batch contrastive objectives with hard negatives improve representation geometry.
We extend this interface to concept-equivalence retrieval and find that semantic signal dominates the final layer in sentence-fine-tuned encoders, limiting the benefit of cross-layer pooling.

\paragraph{Biomedical concept normalisation.}
BioSyn \citep{sung2020biosyn}, SapBERT \citep{liu2021sapbert}, and BioLORD \citep{remy2022biolord} demonstrate that synonym-marginalization and definition-aware contrastive training substantially improve biomedical entity retrieval, and \citet{tutubalina2020fair} show that reported accuracy depends heavily on split design.
Training on dictionary definitions as concept supervision has prior support \citep{hill2016dictionary,carvalho2023definitions}; we use WordNet/Wiktionary synonym and definition pairs in the same vein, but centre our analysis on the mechanism of fine-tuning and introduce controlled modifier-type evaluation absent from prior work.

\paragraph{Distributional composition and modifier sensitivity.}
Compositional distributional semantics \citep{mitchell2010composition,baroni2010nouns} shows that typed composition operators outperform uniform ones and that adjective denotation depends on semantic role---directly motivating the $\varepsilon_\tau$-compositionality framework in \S\ref{sec:prelim}.
The modifier typology underlying our benchmark (intersective, subsective, relational, modal, privative) originates in formal semantics \citep{partee1995lexical}, and \citet{ettinger2018composition}, \citet{shwartz2019nouncompounds}, and \citet{carvalho2025montague} show that these distinctions remain problematic for contemporary encoders.
Our NP paraphrase benchmark operationalises these concerns as a retrieval task with explicit modifier-family labels, to our knowledge the first to stratify retrieval performance across Montague modifier families.

\paragraph{Layer distribution and geometry.}
Probing studies establish that lower layers of pretrained transformers encode syntax while upper layers encode semantics \citep{peters2018elmo,jawahar2019bert,tenney2019bert,rogers2020bertology}, and \citet{ethayarajh2019geometry} show this hierarchy is reflected in anisotropy: upper layers are geometrically more uniform and task-ready.
We find this distribution collapses in sentence-fine-tuned encoders: prior contrastive training has already concentrated semantic signal into the final layer, and concept-equivalence fine-tuning sharpens it further, leaving nothing for cross-layer readout to exploit.
While hyperbolic geometries \citep{nickel2017poincare,valentino2024hyperbolic} offer theoretically richer hierarchical structure, our results confirm that geometry choice is secondary to training supervision once the space is well-calibrated.

\section{Representational Compositionality}
\label{sec:prelim}

\subsection{Problem Setting}

Let $\mathcal{X}$ be a space of texts (terms, noun phrases, definitions, ontology labels).
We learn an encoder $f_\theta : \mathcal{X} \to \mathbb{R}^d$ and a retrieval score
$\mathrm{score}_\theta(q, y) \in \mathbb{R}$.
Given a query $q$ and a candidate pool $\mathcal{D} = \{y_1,\dots,y_N\}$,
the objective is to rank semantically equivalent candidates first:
\[
\begin{aligned}
  \mathrm{score}_\theta(q, y^+) &> \mathrm{score}_\theta(q, y^-)\\
  &\quad\text{whenever } \sem{q} = \sem{y^+} \neq \sem{y^-},
\end{aligned}
\]
where $\sem{\cdot}$ denotes the concept denotation.
In practice, positives are synonym pairs, term-definition pairs, and cross-source
definition pairs from the same concept.
The hard case is a zero-overlap triplet where the correct answer $y^+$ shares no surface tokens with $q$.

\subsection{From Semantic Composition to Representational Composition}
\label{ssec:semantic-composition}

Semantic compositionality alone does not guarantee compositional representations.
Let $f_\theta : \mathcal{X} \to \mathbb{R}^d$ be a sentence encoder.
We say that $f_\theta$ is $\varepsilon_\tau$-compositional for modifier family $\tau$ if there exists a latent operator
\[
  \Phi_\tau : \mathbb{R}^d \times \mathbb{R}^d \to \mathbb{R}^d
\]
such that for all valid modifier--head pairs $(m,h)$ of type $\tau$,
\[
  \left\lVert f_\theta\!\left(m \circ_\tau h\right)
  - \Phi_\tau\!\left(f_\theta(m), f_\theta(h)\right) \right\rVert
  \le \varepsilon_\tau.
\]
This definition makes explicit the bridge between conceptual and geometric compositionality.
Semantic composition is typed, but the encoder must realize all such operators in a shared latent space.
Concept retrieval then depends on low distortion of the relevant $\Phi_\tau$, and a scoring function that ranks denotationally equivalent expressions above non-equivalent distractors.

Appendix~\ref{app:theory} formalizes these definitions and derives theoretical bounds on retrieval distortion for each composition family.

\subsection{Why Pooled Embeddings Struggle}

Following Montague semantics \citep{carvalho2025montague}, the denotation of a complex nominal is:
\[
  \sem{m \circ_\tau h} = C_\tau\!\bigl(\sem{m},\, \sem{h}\bigr),
\]
where $\tau$ is the modifier composition type and $C_\tau$ is a type-specific operator.
Different modifiers types instantiate materially different $C_\tau$.
These five types constitute the standard complete typology of modifier composition under Montague generative semantics \citep{carvalho2025montague}, spanning the full range from extensional set-intersection (intersective) to intensional non-instantiation (privative).
A pooled vector $z(x) = g(H^{(1)},\dots,H^{(L)})$ cannot expose the type $\tau$ explicitly, but encodes it implicitly in the geometry.

This creates two interacting bottlenecks.
The \textbf{representation bottleneck}: structured token-role information is collapsed into one vector, and multiple semantic relations must coexist in one shared neighborhood structure.
The \textbf{objective bottleneck}: standard sentence embedding training rewards broad
distributional semantic similarity, not operator-sensitive conceptual precision.
Our study isolates and investigates these two aspects. Through contrastive fine-tuning on synonym and definition pairs directly, we target the objective bottleneck, and a cross-layer readout ablation tests whether the representation bottleneck matters after fine-tuning.

\section{Proposed Method}

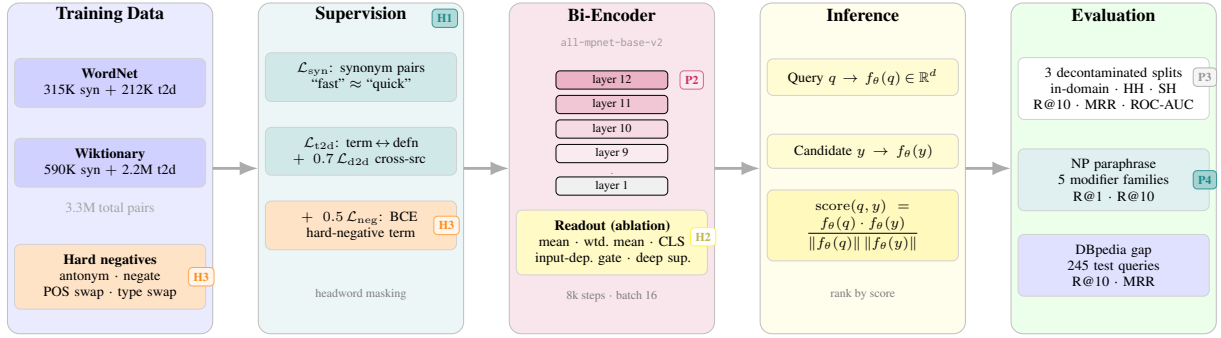
\begin{figure*}[t]
\centering
\resizebox{\linewidth}{!}{%
\begin{tikzpicture}[font=\small, >=Latex,
  ph/.style={draw=gray!50, rounded corners=6pt, fill=#1,
             minimum width=3.6cm, minimum height=5.8cm},
  ph/.default=gray!6,
  it/.style={draw=gray!35, rounded corners=3pt, fill=#1,
             font=\scriptsize, inner sep=5pt, text width=3.0cm, align=center},
  it/.default=white,
  arr/.style={->, line width=1.4pt, gray!65}
]

\node[ph=blue!8] at (0,0) {};
\node[font=\footnotesize\bfseries] at (0, 2.7) {Training Data};
\node[it=blue!18] at (0, 1.5) {\textbf{WordNet}\\315K syn $+$ 212K t2d};
\node[it=blue!18] at (0, 0.1) {\textbf{Wiktionary}\\590K syn $+$ 2.2M t2d};
\node[font=\scriptsize, gray!70, align=center] at (0,-0.7) {3.3M total pairs};
\node[it=orange!22] at (0,-1.9) {\textbf{Hard negatives}\\antonym $\cdot$ negate\\POS swap $\cdot$ type swap};

\node[ph=teal!7] at (4.4,0) {};
\node[font=\footnotesize\bfseries] at (4.4, 2.7) {Supervision};
\node[it=teal!18] at (4.4, 1.6) {$\mathcal{L}_\mathrm{syn}$: synonym pairs\\``fast'' $\approx$ ``quick''};
\node[it=teal!18] at (4.4, 0.3) {$\mathcal{L}_\mathrm{t2d}$: term\,$\leftrightarrow$\,defn\\$+\;0.7\,\mathcal{L}_\mathrm{d2d}$ cross-src};
\node[it=orange!22] at (4.4,-1.0) {$+\;0.5\,\mathcal{L}_\mathrm{neg}$: BCE\\hard-negative term};
\node[font=\tiny, gray, align=center] at (4.4,-2.2) {headword masking};

\node[ph=purple!10] at (8.8,0) {};
\node[font=\footnotesize\bfseries] at (8.8, 2.7) {Bi-Encoder};
\node[font=\tiny, gray!80, align=center] at (8.8, 2.2) {\texttt{all-mpnet-base-v2}};
\foreach \ii/\yy/\cc in {12/1.55/purple!30, 11/1.12/purple!22, 10/0.69/purple!14, 9/0.26/purple!8}{
  \draw[fill=\cc, rounded corners=2pt] (7.82,\yy-0.16) rectangle (9.78,\yy+0.16);
  \node[font=\tiny] at (8.8,\yy) {layer \ii};
}
\node[font=\tiny, gray] at (8.8,-0.12) {$\vdots$};
\draw[fill=gray!12, rounded corners=2pt] (7.82,-0.47) rectangle (9.78,-0.15);
\node[font=\tiny] at (8.8,-0.31) {layer 1};
\node[it=yellow!28] at (8.8,-1.3) {\textbf{Readout (ablation)}\\mean $\cdot$ wtd.\ mean $\cdot$ CLS\\input-dep.\ gate $\cdot$ deep sup.};
\node[font=\tiny, gray, align=center] at (8.8,-2.25) {8k steps $\cdot$ batch 16};

\node[ph=yellow!9] at (13.2,0) {};
\node[font=\footnotesize\bfseries] at (13.2, 2.7) {Inference};
\node[it=yellow!22] at (13.2, 1.6) {Query $q \to f_\theta(q)\!\in\!\mathbb{R}^d$};
\node[it=yellow!22] at (13.2, 0.3) {Candidate $y \to f_\theta(y)$};
\node[it=yellow!32] at (13.2,-1.0) {$\mathrm{score}(q,y) = \dfrac{f_\theta(q)\cdot f_\theta(y)}{\|f_\theta(q)\|\,\|f_\theta(y)\|}$};
\node[font=\tiny, gray, align=center] at (13.2,-2.2) {rank by score};

\node[ph=green!8] at (17.6,0) {};
\node[font=\footnotesize\bfseries] at (17.6, 2.7) {Evaluation};
\node[it] at (17.6, 1.4) {3 decontaminated splits\\in-domain $\cdot$ HH $\cdot$ SH\\R@10 $\cdot$ MRR $\cdot$ ROC-AUC};
\node[it=teal!12] at (17.6,-0.2) {NP paraphrase\\5 modifier families\\R@1 $\cdot$ R@10};
\node[it=blue!12] at (17.6,-1.7) {DBpedia gap\\245 test queries\\R@10 $\cdot$ MRR};

\draw[arr] (1.8,0) -- (2.6,0);
\draw[arr] (6.2,0) -- (7.0,0);
\draw[arr] (10.6,0) -- (11.4,0);
\draw[arr] (15.0,0) -- (15.8,0);

\tikzset{hyp/.style={rounded corners=2pt, font=\tiny\bfseries, inner sep=2.5pt}}
\node[hyp, draw=teal!80, fill=teal!30, text=teal!90, anchor=north east] at (6.1, 2.8) {H1};
\node[hyp, draw=yellow!80, fill=white, text=yellow!70!black] at (10.4, -1.2) {H2};
\node[hyp, draw=orange!70, fill=orange!8, text=orange!90] at (1.6, -1.9) {H3};
\node[hyp, draw=orange!70, fill=orange!8, text=orange!90] at (5.9, -1) {H3};
\node[hyp, draw=purple!60, fill=purple!8, text=purple!80] at (10.2, 1.55) {P2};
\node[hyp, draw=gray!60, fill=gray!8, text=gray!80] at (19.2, 1.6) {P3};
\node[hyp, draw=teal!80, fill=teal!35, text=teal!90] at (19.2, -0.2) {P4};

\end{tikzpicture}%
}
\caption{End-to-end pipeline. WordNet/Wiktionary synonym and definition pairs (3.3M) are combined with five hard-negative rule types in a joint InfoNCE + BCE objective, training a bi-encoder (\texttt{all-mpnet-base-v2}) under eight ablation conditions varying readout strategy and hard-negative supervision.
Evaluation covers three decontaminated term-definition splits (in-domain, HH, SH), modifier-sensitive NP paraphrase retrieval, and zero-lexical-overlap DBpedia property retrieval.
Coloured badges indicate which component each hypothesis and principle targets.}
\label{fig:pipeline}
\end{figure*}

\subsection{Model and Conditions}

The models used in our ablations are based on the bi-encoder sentence-transformer architecture \citep{reimers2019sentencebert} where both query and candidate texts are encoded independently, and similarity is computed via cosine distance.
We define readout (how a single vector is extracted from the transformer's hidden states) as the primary design variable.
Table~\ref{tab:conditions} defines the eight conditions evaluated.

\begin{table}[t]
\centering
\small
\setlength{\tabcolsep}{3pt}
\begin{tabular}{llcl}
\toprule
ID & Pooling & HN & Note \\
\midrule
B0    & Mean (frozen)        &            & Frozen baseline \\
B0-WM & Wtd.\ mean (frozen)  &            & Frozen wtd.\ mean \\
B1    & Mean                 & \checkmark & Fine-tuned ref. \\
B2    & CLS token            & \checkmark & \\
B3    & Weighted mean        &            & Pos.\ only \\
M1    & Weighted mean        & \checkmark & \\
M2    & Input-dep.\ gate     & \checkmark & MLP-gated \\
M3    & Deep supervision     & \checkmark & Per-layer heads \\
\bottomrule
\end{tabular}
\caption{Encoder conditions (shared backbone: \texttt{all-mpnet-base-v2}, 8k steps, batch 16, single seed).
B0/B0-WM are frozen and differ only in pooling, isolating cross-layer averaging without fine-tuning.
Weighted mean pools the last $K{=}4$ layers; M2 uses a per-input MLP gate; M3 attaches per-layer projection heads during training.}
\label{tab:conditions}
\end{table}

The full eight-condition ablation is conducted on a single backbone, \texttt{all-mpnet-base-v2}, to isolate the contribution of each readout and supervision design choice without conflating backbone effects.
To test whether the layerwise structure and the B0$\to$B1 gain generalise beyond this choice, we additionally train B0 and B1 on three further backbones: \texttt{paraphrase-mpnet-base-v2} (same MPNet architecture, paraphrase-oriented pre-training objective), \texttt{e5-base-v2} (a different model family with strong general-purpose text embeddings), and \texttt{jina-embeddings-v5-text-nano} (PEFT-based, trained with a distinct contrastive objective outside the standard sentence-transformer paradigm).
Together, these four backbones provide same-architecture/different-objective variation (\texttt{all-mpnet} vs.\ \texttt{para-mpnet}), cross-family generalisation (\texttt{e5-base}), and a deliberate architectural contrast case (\texttt{jina-nano}).

\subsection{Cross-Layer Readout}

Let $H^{(l)}(x) \in \mathbb{R}^{T \times d}$ be the token hidden states at layer $l$
for an input of length $T$.
The B1 baseline uses standard mean pooling over the final layer only.
The cross-layer variants instead form a weighted mixture over the last $K{=}4$ layers
before pooling:
\[
\begin{aligned}
  z(x) &= P\!\left(\sum_{l=L-K+1}^{L} \alpha_l\, H^{(l)}(x)\right), \\
  \alpha &= \mathrm{softmax}(w),
\end{aligned}
\]
where $P(\cdot)$ denotes mean pooling, $w \in \mathbb{R}^K$ are learned scalar weights
shared across all inputs (B3, M1), and $z(x)$ is L2-normalised before scoring.

\noindent\textbf{M2 (input-dependent gating).}
The global weights $w$ are replaced by a per-input MLP:
\[
  \alpha(x) = \mathrm{softmax}\!\left(\mathrm{MLP}\!\left(P(H^{(L)}(x))\right)\right),
\]
so different inputs can weight layers differently.

\noindent\textbf{M3 (deep supervision).}
During training, a separate linear projection head $\phi_l$ is attached to each of
the last $K$ layers, and an InfoNCE loss is computed at every layer independently:
\[
  \mathcal{L}_{\mathrm{deep}} = \sum_{l=L-K+1}^{L}
    \mathcal{L}_{\mathrm{InfoNCE}}\!\left(\phi_l\!\left(P(H^{(l)})\right)\right).
\]
At inference the projection heads are discarded and the standard weighted-mean
readout is used, forcing all layers to develop useful representations during training
rather than letting the gradient concentrate at the final layer.

\subsection{Training Objective}

The positive loss is a weighted InfoNCE sum over three supervision views:
\[
\mathcal{L}_{\mathrm{pos}}
= \mathcal{L}_{\mathrm{syn}} + \mathcal{L}_{\mathrm{t2d}} + 0.7\,\mathcal{L}_{\mathrm{d2d}},
\]
where $\mathcal{L}_{\mathrm{syn}}$, $\mathcal{L}_{\mathrm{t2d}}$, and
$\mathcal{L}_{\mathrm{d2d}}$ are InfoNCE losses over synonym pairs, term-definition pairs,
and cross-source definition-definition pairs respectively.
The 0.7 weight on d2d reflects its noisier cross-source alignment signal.

When hard negatives are used (all conditions except B0 and B3), a binary cross-entropy term
targets five rule-based negative types per definition
(see Appendix~\ref{app:hardnegs} for full rule definitions and examples).
The complete objective is:
\[
\mathcal{L} = \mathcal{L}_{\mathrm{pos}} + 0.5\,\mathcal{L}_{\mathrm{neg}},
\]
where the 0.5 coefficient moderates the contribution of the hard-negative term relative to the positive InfoNCE objective.
Unlike the batch-level InfoNCE loss, the BCE term applies a gradient of fixed magnitude to each negative pair independently of batch size, without down-weighting, its aggregate contribution would scale with the number of negative pairs and risk overshadowing the ranking objective.
The coefficient was not ablated, however Appendix~\ref{app:loss} provides evidence of its importance by evaluating a variant (B1-unified) in which hard negatives are folded into the InfoNCE denominator, eliminating the need for an explicit weight at the cost of a substantial reduction in calibration.

\subsection{Training Details}

Full hyperparameters are listed in Appendix~\ref{app:hparams}.
All experiments use AdamW with linear warmup (100 steps) and linear decay to zero.
A layerwise LR decay of 0.90 per transformer layer mitigates catastrophic forgetting \citep{howard2018ulmfit}; the pooling head and cross-layer weights receive $5\times$ the base rate.

\subsection{Data}

\paragraph{Training.}
3.3 million synonym and term-to-definition pairs from \textbf{WordNet} (527K: 315K syn + 212K t2d) and \textbf{Wiktionary} (2.8M: 590K syn + 2.2M t2d).
Cross-source definition--definition pairs (d2d) are generated on-the-fly by grouping definitions of the same term across both sources (up to 6 per term; not counted in the 3.3M total).
A headword-masking transform removes the defined term from each definition before training to prevent lexical shortcutting. Representative examples of all three pair types are given in Appendix~\ref{app:data_examples}.

\paragraph{Evaluation splits.}
Three decontaminated splits, each with test concepts strictly absent from training.
\textbf{In-domain} (t2d / d2t): held-out WordNet/Wiktionary pairs partitioned at the synset level.
\textbf{Head-holdout} (HH): the most frequent synsets withheld entirely from training, testing generalisation to high-frequency unseen concepts.
\textbf{Source-holdout} (SH): trained on Wiktionary, evaluated on WordNet, testing cross-source generalisation.
Each split includes hard-negative stress tests (Appendix~\ref{app:hardnegs}).

\paragraph{DBpedia semantic-gap benchmark.}
A new evaluation dataset of (query, property label) pairs over the DBpedia ontology, consisting of 3,063 training queries generated by GPT-4o-mini with zero lexical overlap with property labels, and 245 manually reviewed test pairs with strictly disjoint train/test properties.

\paragraph{NP paraphrase benchmark.}
A new modifier-sensitive evaluation suite of 4,000 noun-phrase paraphrase pairs balanced across five composition families, namely intersective (900), subsective (800), modal (800), privative (800), relational (700). These cover the complete Montague modifier typology (Section~\ref{sec:prelim}).
Pairs are drawn from two equivalence regimes of 2,000 each, to create \textbf{strict} (GPT-4o-mini-generated paraphrases with zero lexical overlap enforced after head masking) and \textbf{near} (rule-based head-broadening or relation-abstraction transforms applied to the strict pairs).

\section{Empirical Analysis}
\label{sec:results}


\subsection{Layer-12 Dominance: The Mechanism Behind H2}
\label{ssec:mech_exp}

\begin{figure*}[t]
\centering
\resizebox{\linewidth}{!}{%
\begin{tikzpicture}[font=\small, >=Latex]
  \definecolor{mpnetC}{RGB}{72,120,207}
  \definecolor{pmpnetC}{RGB}{60,150,60}
  \definecolor{e5C}{RGB}{210,115,0}
  \definecolor{jinaC}{RGB}{155,0,210}
  \definecolor{miniT}{RGB}{0,155,150}

  \begin{scope}
    \draw[->] (0,0) -- (0,5.8);
    \draw[->] (0,0) -- (9.0,0);
    \node[anchor=south, font=\footnotesize\bfseries] at (4.25,5.8) {Retrieval quality (MRR) by layer};
    \node[rotate=90, anchor=south, font=\footnotesize] at (-0.7,2.6) {MRR};
    \node[anchor=north, font=\footnotesize] at (4.25,-0.5) {Layer index};
    \foreach \y/\ylab in {0.35/0.42, 1.05/0.44, 1.75/0.46, 2.45/0.48, 3.15/0.50, 3.85/0.52, 4.55/0.54, 5.25/0.56} {
      \draw[gray!25] (0,\y) -- (8.5,\y);
      \draw (-0.05,\y) -- (0,\y);
      \node[anchor=east, font=\scriptsize] at (-0.08,\y) {\ylab};
    }
    \foreach \xpos/\layer in {0.5/5, 1.5/6, 2.5/7, 3.5/8, 4.5/9, 5.5/10, 6.5/11, 7.5/12} {
      \node[anchor=north, font=\scriptsize] at (\xpos,-0.08) {\layer};
    }
    \draw[mpnetC, line width=1.6pt]
      (0.5,0.98) -- (1.5,1.26) -- (2.5,1.37) -- (3.5,1.68) -- (4.5,1.82) -- (5.5,2.59) -- (6.5,3.68) -- (7.5,4.03);
    \foreach \xpos/\ypos in {0.5/0.98,1.5/1.26,2.5/1.37,3.5/1.68,4.5/1.82,5.5/2.59,6.5/3.68,7.5/4.03}
      \fill[mpnetC] (\xpos,\ypos) circle (0.09);
    \draw[mpnetC, line width=1.6pt, dashed]
      (0.5,0.90) -- (1.5,0.78) -- (2.5,0.68) -- (3.5,0.76) -- (4.5,0.97) -- (5.5,2.45) -- (6.5,3.93) -- (7.5,4.24);
    \foreach \xpos/\ypos in {0.5/0.90,1.5/0.78,2.5/0.68,3.5/0.76,4.5/0.97,5.5/2.45,6.5/3.93,7.5/4.24}
      \fill[mpnetC] (\xpos,\ypos) circle (0.09);
    \draw[pmpnetC, line width=1.4pt]
      (0.5,1.09) -- (1.5,1.33) -- (2.5,1.44) -- (3.5,1.37) -- (4.5,1.51) -- (5.5,1.68) -- (6.5,3.68) -- (7.5,4.90);
    \foreach \xpos/\ypos in {0.5/1.09,1.5/1.33,2.5/1.44,3.5/1.37,4.5/1.51,5.5/1.68,6.5/3.68,7.5/4.90}
      \fill[pmpnetC] (\xpos,\ypos+0.12) -- (\xpos-0.10,\ypos-0.07) -- (\xpos+0.10,\ypos-0.07) -- cycle;
    \draw[e5C, line width=1.4pt]
      (0.5,0.95) -- (1.5,0.49) -- (2.5,0.74) -- (3.5,1.03) -- (4.5,0.93) -- (5.5,2.08) -- (6.5,3.89) -- (7.5,5.34);
    \foreach \xpos/\ypos in {0.5/0.95,1.5/0.49,2.5/0.74,3.5/1.03,4.5/0.93,5.5/2.08,6.5/3.89,7.5/5.34}
      {\draw[e5C,line width=1.0pt] (\xpos-0.09,\ypos-0.09)--(\xpos+0.09,\ypos+0.09);
       \draw[e5C,line width=1.0pt] (\xpos+0.09,\ypos-0.09)--(\xpos-0.09,\ypos+0.09);}
    \draw[pmpnetC, line width=1.4pt, dashed]
      (0.5,1.02) -- (1.5,1.05) -- (2.5,1.12) -- (3.5,1.20) -- (4.5,1.22) -- (5.5,1.95) -- (6.5,3.85) -- (7.5,5.25);
    \foreach \xpos/\ypos in {0.5/1.02,1.5/1.05,2.5/1.12,3.5/1.20,4.5/1.22,5.5/1.95,6.5/3.85,7.5/5.25}
      \fill[pmpnetC] (\xpos,\ypos+0.12) -- (\xpos-0.10,\ypos-0.07) -- (\xpos+0.10,\ypos-0.07) -- cycle;
    \draw[e5C, line width=1.4pt, dashed]
      (0.5,0.75) -- (1.5,0.56) -- (2.5,0.58) -- (3.5,0.60) -- (4.5,0.65) -- (5.5,2.70) -- (6.5,4.15) -- (7.5,5.50);
    \foreach \xpos/\ypos in {0.5/0.75,1.5/0.56,2.5/0.58,3.5/0.60,4.5/0.65,5.5/2.70,6.5/4.15,7.5/5.50}
      {\draw[e5C,line width=1.0pt] (\xpos-0.09,\ypos-0.09)--(\xpos+0.09,\ypos+0.09);
       \draw[e5C,line width=1.0pt] (\xpos+0.09,\ypos-0.09)--(\xpos-0.09,\ypos+0.09);}
    \draw[jinaC, line width=1.4pt]
      (0.5,0.82) -- (1.5,0.99) -- (2.5,1.20) -- (3.5,1.32) -- (4.5,1.45) -- (5.5,2.09) -- (6.5,3.90) -- (7.5,4.42);
    \foreach \xpos/\ypos in {0.5/0.82,1.5/0.99,2.5/1.20,3.5/1.32,4.5/1.45,5.5/2.09,6.5/3.90,7.5/4.42}
      \fill[jinaC] (\xpos,\ypos-0.12)--(\xpos-0.10,\ypos+0.07)--(\xpos+0.10,\ypos+0.07)--cycle;
    \draw[jinaC, line width=1.4pt, dashed]
      (0.5,0.92) -- (1.5,1.05) -- (2.5,1.22) -- (3.5,1.45) -- (4.5,1.80) -- (5.5,2.67) -- (6.5,4.29) -- (7.5,4.60);
    \foreach \xpos/\ypos in {0.5/0.92,1.5/1.05,2.5/1.22,3.5/1.45,4.5/1.80,5.5/2.67,6.5/4.29,7.5/4.60}
      \fill[jinaC] (\xpos,\ypos-0.12)--(\xpos-0.10,\ypos+0.07)--(\xpos+0.10,\ypos+0.07)--cycle;
    \fill[white,opacity=0.9] (0.05,4.18) rectangle (6.10,5.75);
    \draw[mpnetC,line width=1.4pt] (0.2,5.50)--(1.0,5.50);
    \fill[mpnetC] (0.60,5.50) circle (0.09);
    \node[anchor=west,font=\scriptsize] at (1.05,5.50) {B0 all-mpnet};
    \draw[mpnetC,line width=1.4pt,dashed] (3.2,5.50)--(4.0,5.50);
    \fill[mpnetC] (3.60,5.50) circle (0.09);
    \node[anchor=west,font=\scriptsize] at (4.05,5.50) {B1 all-mpnet};
    \draw[pmpnetC,line width=1.4pt] (0.2,5.12)--(1.0,5.12);
    \fill[pmpnetC] (0.60,5.24)--(0.50,5.05)--(0.70,5.05)--cycle;
    \node[anchor=west,font=\scriptsize] at (1.05,5.12) {B0 para-mpnet};
    \draw[pmpnetC,line width=1.4pt,dashed] (3.2,5.12)--(4.0,5.12);
    \fill[pmpnetC] (3.60,5.24)--(3.50,5.05)--(3.70,5.05)--cycle;
    \node[anchor=west,font=\scriptsize] at (4.05,5.12) {B1 para-mpnet};
    \draw[e5C,line width=1.4pt] (0.2,4.74)--(1.0,4.74);
    \draw[e5C,line width=1.0pt] (0.51,4.65)--(0.69,4.83);
    \draw[e5C,line width=1.0pt] (0.69,4.65)--(0.51,4.83);
    \node[anchor=west,font=\scriptsize] at (1.05,4.74) {B0 e5-base};
    \draw[e5C,line width=1.4pt,dashed] (3.2,4.74)--(4.0,4.74);
    \draw[e5C,line width=1.0pt] (3.51,4.65)--(3.69,4.83);
    \draw[e5C,line width=1.0pt] (3.69,4.65)--(3.51,4.83);
    \node[anchor=west,font=\scriptsize] at (4.05,4.74) {B1 e5-base};
    \draw[jinaC,line width=1.4pt] (0.2,4.36)--(1.0,4.36);
    \fill[jinaC] (0.60,4.24)--(0.50,4.43)--(0.70,4.43)--cycle;
    \node[anchor=west,font=\scriptsize] at (1.05,4.36) {B0 jina-nano};
    \draw[jinaC,line width=1.4pt,dashed] (3.2,4.36)--(4.0,4.36);
    \fill[jinaC] (3.60,4.24)--(3.50,4.43)--(3.70,4.43)--cycle;
    \node[anchor=west,font=\scriptsize] at (4.05,4.36) {B1 jina-nano};
  \end{scope}

  \begin{scope}[xshift=10.5cm]
    \draw[->] (0,0) -- (0,5.8);
    \draw[->] (0,0) -- (9.0,0);
    \node[anchor=south, font=\footnotesize\bfseries] at (4.25,5.8) {Geometry (anisotropy) by layer};
    \node[rotate=90, anchor=south, font=\footnotesize] at (-0.7,2.6) {Anisotropy};
    \node[anchor=north, font=\footnotesize] at (4.25,-0.5) {Layer index};
    \foreach \y/\ylab in {0.11/0.10, 1.21/0.20, 2.31/0.30, 3.41/0.40, 4.51/0.50} {
      \draw[gray!25] (0,\y) -- (8.5,\y);
      \draw (-0.05,\y) -- (0,\y);
      \node[anchor=east, font=\scriptsize] at (-0.08,\y) {\ylab};
    }
    \foreach \xpos/\layer in {0.5/5, 1.5/6, 2.5/7, 3.5/8, 4.5/9, 5.5/10, 6.5/11, 7.5/12} {
      \node[anchor=north, font=\scriptsize] at (\xpos,-0.08) {\layer};
    }
    \draw[mpnetC, line width=1.6pt]
      (0.5,4.19) -- (1.5,4.50) -- (2.5,4.08) -- (3.5,3.99) -- (4.5,4.13) -- (5.5,4.41) -- (6.5,4.04) -- (7.5,0.62);
    \foreach \xpos/\ypos in {0.5/4.19,1.5/4.50,2.5/4.08,3.5/3.99,4.5/4.13,5.5/4.41,6.5/4.04,7.5/0.62}
      \fill[mpnetC] (\xpos,\ypos) circle (0.09);
    \draw[mpnetC, line width=1.6pt, dashed]
      (0.5,4.44) -- (1.5,4.77) -- (2.5,4.47) -- (3.5,4.40) -- (4.5,4.39) -- (5.5,4.28) -- (6.5,3.47) -- (7.5,0.30);
    \foreach \xpos/\ypos in {0.5/4.44,1.5/4.77,2.5/4.47,3.5/4.40,4.5/4.39,5.5/4.28,6.5/3.47,7.5/0.30}
      \fill[mpnetC] (\xpos,\ypos) circle (0.09);
    \draw[pmpnetC, line width=1.4pt]
      (0.5,4.68) -- (1.5,4.81) -- (2.5,4.38) -- (3.5,4.31) -- (4.5,4.29) -- (5.5,4.60) -- (6.5,4.28) -- (7.5,0.77);
    \foreach \xpos/\ypos in {0.5/4.68,1.5/4.81,2.5/4.38,3.5/4.31,4.5/4.29,5.5/4.60,6.5/4.28,7.5/0.77}
      \fill[pmpnetC] (\xpos,\ypos+0.12) -- (\xpos-0.10,\ypos-0.07) -- (\xpos+0.10,\ypos-0.07) -- cycle;
    \draw[jinaC, line width=1.4pt]
      (0.5,4.12) -- (1.5,3.89) -- (2.5,3.38) -- (3.5,3.57) -- (4.5,3.35) -- (5.5,4.82) -- (6.5,5.42) -- (7.5,2.88);
    \foreach \xpos/\ypos in {0.5/4.12,1.5/3.89,2.5/3.38,3.5/3.57,4.5/3.35,5.5/4.82,6.5/5.42,7.5/2.88}
      \fill[jinaC] (\xpos,\ypos-0.12)--(\xpos-0.10,\ypos+0.07)--(\xpos+0.10,\ypos+0.07)--cycle;
    \draw[pmpnetC, line width=1.4pt, dashed]
      (0.5,4.78) -- (1.5,4.75) -- (2.5,4.66) -- (3.5,4.55) -- (4.5,4.43) -- (5.5,4.30) -- (6.5,3.70) -- (7.5,0.66);
    \foreach \xpos/\ypos in {0.5/4.78,1.5/4.75,2.5/4.66,3.5/4.55,4.5/4.43,5.5/4.30,6.5/3.70,7.5/0.66}
      \fill[pmpnetC] (\xpos,\ypos+0.12) -- (\xpos-0.10,\ypos-0.07) -- (\xpos+0.10,\ypos-0.07) -- cycle;
    \draw[jinaC, line width=1.4pt, dashed]
      (0.5,4.17) -- (1.5,4.00) -- (2.5,4.09) -- (3.5,4.22) -- (4.5,4.11) -- (5.5,4.50) -- (6.5,4.80) -- (7.5,2.71);
    \foreach \xpos/\ypos in {0.5/4.17,1.5/4.00,2.5/4.09,3.5/4.22,4.5/4.11,5.5/4.50,6.5/4.80,7.5/3.10}
      \fill[jinaC] (\xpos,\ypos-0.12)--(\xpos-0.10,\ypos+0.07)--(\xpos+0.10,\ypos+0.07)--cycle;
    \draw[e5C, line width=1.4pt]
      (0.5,4.82) -- (1.5,4.91) -- (2.5,4.87) -- (3.5,4.71) -- (4.5,4.50) -- (5.5,4.55) -- (6.5,3.69) -- (7.5,1.90);
    \foreach \xpos/\ypos in {0.5/4.82,1.5/4.91,2.5/4.87,3.5/4.71,4.5/4.50,5.5/4.55,6.5/3.69,7.5/1.90}
      {\draw[e5C,line width=1.0pt] (\xpos-0.09,\ypos-0.09)--(\xpos+0.09,\ypos+0.09);
       \draw[e5C,line width=1.0pt] (\xpos+0.09,\ypos-0.09)--(\xpos-0.09,\ypos+0.09);}
    \draw[e5C, line width=1.4pt, dashed]
      (0.5,4.88) -- (1.5,4.99) -- (2.5,4.95) -- (3.5,4.85) -- (4.5,4.72) -- (5.5,4.65) -- (6.5,3.52) -- (7.5,1.33);
    \foreach \xpos/\ypos in {0.5/4.88,1.5/4.99,2.5/4.95,3.5/4.85,4.5/4.72,5.5/4.65,6.5/3.52,7.5/1.33}
      {\draw[e5C,line width=1.0pt] (\xpos-0.09,\ypos-0.09)--(\xpos+0.09,\ypos+0.09);
       \draw[e5C,line width=1.0pt] (\xpos+0.09,\ypos-0.09)--(\xpos-0.09,\ypos+0.09);}
    \fill[white,opacity=0.9] (0.05,0.18) rectangle (6.10,1.68);
    \draw[mpnetC,line width=1.4pt] (0.2,1.50)--(1.0,1.50);
    \fill[mpnetC] (0.60,1.50) circle (0.09);
    \node[anchor=west,font=\scriptsize] at (1.05,1.50) {B0 all-mpnet};
    \draw[mpnetC,line width=1.4pt,dashed] (3.2,1.50)--(4.0,1.50);
    \fill[mpnetC] (3.60,1.50) circle (0.09);
    \node[anchor=west,font=\scriptsize] at (4.05,1.50) {B1 all-mpnet};
    \draw[pmpnetC,line width=1.4pt] (0.2,1.12)--(1.0,1.12);
    \fill[pmpnetC] (0.60,1.24)--(0.50,1.05)--(0.70,1.05)--cycle;
    \node[anchor=west,font=\scriptsize] at (1.05,1.12) {B0 para-mpnet};
    \draw[pmpnetC,line width=1.4pt,dashed] (3.2,1.12)--(4.0,1.12);
    \fill[pmpnetC] (3.60,1.24)--(3.50,1.05)--(3.70,1.05)--cycle;
    \node[anchor=west,font=\scriptsize] at (4.05,1.12) {B1 para-mpnet};
    \draw[jinaC,line width=1.4pt] (0.2,0.74)--(1.0,0.74);
    \fill[jinaC] (0.60,0.62)--(0.50,0.81)--(0.70,0.81)--cycle;
    \node[anchor=west,font=\scriptsize] at (1.05,0.74) {B0 jina-nano};
    \draw[jinaC,line width=1.4pt,dashed] (3.2,0.74)--(4.0,0.74);
    \fill[jinaC] (3.60,0.62)--(3.50,0.81)--(3.70,0.81)--cycle;
    \node[anchor=west,font=\scriptsize] at (4.05,0.74) {B1 jina-nano};
    \draw[e5C,line width=1.4pt] (0.2,0.36)--(1.0,0.36);
    \draw[e5C,line width=1.0pt] (0.51,0.27)--(0.69,0.45); \draw[e5C,line width=1.0pt] (0.69,0.27)--(0.51,0.45);
    \node[anchor=west,font=\scriptsize] at (1.05,0.36) {B0 e5-base};
    \draw[e5C,line width=1.4pt,dashed] (3.2,0.36)--(4.0,0.36);
    \draw[e5C,line width=1.0pt] (3.51,0.27)--(3.69,0.45); \draw[e5C,line width=1.0pt] (3.69,0.27)--(3.51,0.45);
    \node[anchor=west,font=\scriptsize] at (4.05,0.36) {B1 e5-base};
  \end{scope}
\end{tikzpicture}%
}
\caption{Per-layer MRR (left) and anisotropy (right) for transformer layers 5--12 on NP paraphrase pairs (solid = B0 frozen, dashed = B1 fine-tuned).
Final-layer MRR dominance pre-exists fine-tuning across all tested backbones: layer-12 MRR far exceeds earlier layers, with anisotropy dropping sharply at the final layer.
Fine-tuning sharpens both effects without redistributing signal to earlier layers, leaving nothing for cross-layer readout to exploit.}
\label{fig:layerwise}
\end{figure*}
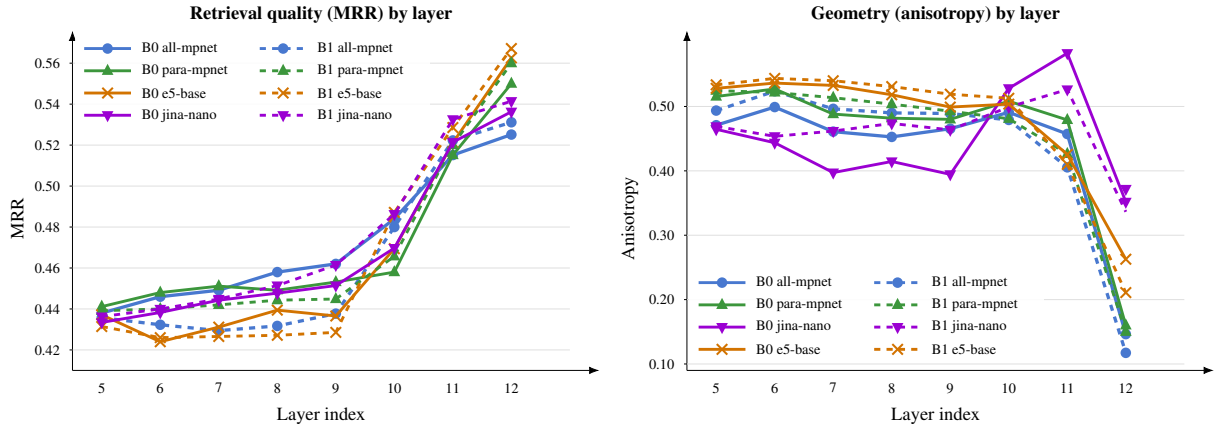

We present the layerwise analysis before the H1 ablation to establish the mechanistic explanation that directly predicts the H2 null result and motivates the $K{=}4$ window choice in the cross-layer conditions.
B1 results appear here alongside B0 for contrast; that B1 represents a substantial retrieval improvement over B0 is confirmed in §\ref{ssec:h1_exp}.

Figure~\ref{fig:layerwise} shows that the final layer dominates in both B0 and B1. MRR at layer 12 far exceeds layers 5--7, while anisotropy drops sharply at the last two layers.
Fine-tuning sharpens this concentration, in that earlier layers see lower MRR and higher anisotropy after training, but training does not redistribute information across layers.
The pattern holds across all tested sentence-transformer backbones.

Probing studies \citep{jawahar2019bert,tenney2019bert} established the layer hierarchy on raw pretrained models, where signal \emph{is} distributed and cross-layer readout could plausibly help.
In an already-sentence-fine-tuned encoder, that distribution collapses before our fine-tuning even begins, and concept-equivalence training sharpens it further.
There is no distributed multi-layer signal for a cross-layer readout to exploit. The H2 ablation (Table~\ref{tab:h2}, §\ref{ssec:h2_exp}) confirms this directly.

\subsection{H1: Fine-Tuning Substantially Improves Concept Retrieval}
\label{ssec:h1_exp}

Table~\ref{tab:h1} shows that B0 (frozen \texttt{all-mpnet-base-v2} with mean pooling) achieves t2d R@10 = 0.552 and negate-stress ROC-AUC = 0.470 (below chance). We highlight this failure mode in the introduction, where the frozen encoder assigns \emph{higher} similarity to negated definitions than to correct ones.
Concept-equivalence fine-tuning resolves both, as B1 reaches t2d R@10 = 0.654 and negate ROC-AUC = 0.980, confirming that supervision is necessary and that it specifically addresses the negation failure.
However, fine-tuning is not uniformly beneficial, given that B0 already achieves NP paraphrase R@10 = 0.704, exceeding all fine-tuned conditions (B1: 0.656, B3: 0.655).
The backbone's prior sentence-level training is well-suited to paraphrase retrieval. Concept-equivalence fine-tuning improves cross-modal concept matching precisely where the backbone is weak, at the cost of paraphrase-level similarity where it was already well-calibrated.

The B3 vs.\ M1 comparison reveals that B3 matches M1 on retrieval ranking (0.657 vs.\ 0.651 t2d R@10) but is far weaker on calibration (pair ROC-AUC 0.769 vs.\ 0.961) and semantic stress tests (negate ROC-AUC 0.528 vs.\ 0.984).
Hard negatives train the model to \emph{discriminate}, not to \emph{rank}. Whether hard negatives are necessary depends entirely on the downstream use case.

\begin{table*}[t]
\centering
\small
\setlength{\tabcolsep}{5pt}
\begin{tabular}{llrrrrrr}
\toprule
ID & Condition & t2d R@10 & t2d MRR & Pair ROC-AUC & HH d2t R@10 & Negate & NP R@10 \\
\midrule
B0    & Frozen baseline (mean)         & 0.552 & 0.432 & 0.718 & 0.546 & 0.470 & 0.704 \\
B0-WM & Frozen baseline (wtd.\ mean)   & 0.392 & 0.303 & 0.690 & 0.438 & 0.461 & 0.675 \\
B1 & Mean pool + hard neg.           & 0.654 & 0.537 & \textbf{0.963} & \textbf{0.686} & 0.980 & \textbf{0.656} \\
B3 & Weighted mean, no hard neg.     & \textbf{0.657} & \textbf{0.544} & 0.769
 & 0.677 & 0.528 & 0.655 \\
M1 & Weighted mean + hard neg.       & 0.651 & 0.532 & 0.961 & 0.673 & \textbf{0.984} & 0.623 \\
\bottomrule
\end{tabular}
\caption{H1 ablation results. t2d = term-to-definition (in-domain); HH = head-holdout; Negate = negation stress-test ROC-AUC; NP R@10 = noun-phrase paraphrase ($n$=4,000). Bold: best fine-tuned condition per column.}
\label{tab:h1}
\end{table*}

Fine-tuning collapses anisotropy from 0.126 to 0.012 while leaving effective rank essentially unchanged (247.0 $\to$ 247.2), confirming that concept-equivalence training reorganises the representation space rather than expanding it. Hard negatives are the dominant factor, given that B3 without hard negatives reaches only 0.166 anisotropy under an otherwise identical regime.
Full geometry statistics, condition-level comparisons, and an anisotropy--effective-rank scatter are in Appendix~\ref{app:geometry} (Figures~\ref{fig:anisotropy_comparison} and~\ref{fig:anisotropy_scatter}).

\subsection{H2: Cross-Layer Pooling Offers No Consistent Benefit}
\label{ssec:h2_exp}

Table~\ref{tab:h2} reveals cross-layer pooling to be redundant.
No pooling variant consistently outperforms the fine-tuned mean-pool baseline B1.
CLS pooling (B2) and the wider K=6 window underperform B1 on all splits.
Input-dependent gating (M2) is the only variant that matches or marginally exceeds B1
($+$0.005 t2d R@10), and the margin is $<$1\% absolute.
Deep supervision (M3) improves over the static weighted mixture M1 but does not reach B1.
This is the expected consequence of the layer-concentration principle (Section~\ref{ssec:mech_exp}).

\begin{table}[t]
\centering
\small
\setlength{\tabcolsep}{4pt}
\begin{tabular}{lrrrr}
\toprule
 & \multicolumn{2}{c}{In-domain} & \multicolumn{2}{c}{HH} \\
\cmidrule(lr){2-3}\cmidrule(lr){4-5}
 & t2d R@10 & d2t R@10 & t2d R@10 & d2t R@10 \\
\midrule
B1 & 0.654 & 0.661 & 0.673 & \textbf{0.686} \\
B2 & 0.650 & 0.657 & 0.666 & 0.679 \\
M1 & 0.651 & 0.655 & 0.665 & 0.673 \\
M2 & \textbf{0.659} & \textbf{0.664} & \textbf{0.676} & \textbf{0.688} \\
M3 & 0.649 & 0.654 & 0.667 & 0.674 \\
\bottomrule
\end{tabular}
\caption{H2 retrieval results (R@10). HH = head-holdout; t2d = term-to-definition; d2t = definition-to-term. No cross-layer variant consistently beats B1; M2 matches or marginally exceeds it by $<$1\% absolute. Bold: best per column.}
\label{tab:h2}
\end{table}

\subsection{H3: Hard Negatives Improve Calibration, Not Ranking}
\label{ssec:h3_exp}

\begin{table}[t]
\centering
\small
\setlength{\tabcolsep}{4pt}
\begin{tabular}{lrrrr}
\toprule
 & \multicolumn{2}{c}{No hard negatives} & \multicolumn{2}{c}{Hard negatives} \\
\cmidrule(lr){2-3}\cmidrule(lr){4-5}
Metric & B0 & B3 & B1 & M1 \\
\midrule
t2d R@10        & 0.552 & \textbf{0.657} & 0.654 & 0.651 \\
Pair ROC-AUC    & 0.718 & 0.769 & \textbf{0.963} & 0.961 \\
\midrule
Antonym flip    & 0.894 & 0.946 & \textbf{0.965} & 0.964 \\
Negate          & 0.470 & 0.528 & \textbf{0.986} & 0.984 \\
Same-POS (rnd.) & 0.903 & \textbf{0.962} & \textbf{0.962} & \textbf{0.962} \\
Same-POS (pfx.) & 0.856 & 0.919 & 0.926 & \textbf{0.928} \\
Type swap       & 0.453 & 0.472 & \textbf{0.981} & 0.976 \\
\bottomrule
\end{tabular}
\caption{H3 stress-test results (ROC-AUC, in-domain). Upper: retrieval ranking and calibration; lower: per-type stress tests. B3 vs.\ M1 isolates hard-negative supervision (same pooling and budget). Bold: best per row.}
\label{tab:h3}
\end{table}

Table~\ref{tab:h3} shows that hard negatives improve discrimination, not ranking.
B3 and M1 provide the cleanest isolation, differing only in whether the BCE hard-negative term is active.
B3 achieves t2d R@10 = 0.657 vs.\ M1's 0.651 ($-$0.006, effectively unchanged), while negate ROC-AUC jumps from 0.528 to 0.984 ($+$0.456) and type-swap from 0.472 to 0.976 ($+$0.504). Ranking and calibration are separable properties of the concept representation space, governed by different components of the training objective.


We conclude that hard negatives are warranted when the downstream task requires semantic discrimination. These tasks include pair scoring, similarity thresholding, and robustness to adversarial paraphrase.
For ranking-only pipelines, where the goal is Recall@K rather than a calibrated similarity score, B3 matches or exceeds M1, and the additional complexity of rule-based hard-negative generation and BCE weighting is unnecessary.

\subsection{Modifier-Family Analysis: Extensional Supervision Helps Extensional Families}

Table~\ref{tab:modifiers} breaks down NP Recall@1 by modifier type, revealing a pattern that mirrors the semantic type hierarchy in Section~\ref{sec:prelim} (see Appendix~\ref{app:modifier_types} for a plain-language description of each family).
Extensional families (intersective, subsective) allow the correct paraphrase to be recovered from proximity to the head-noun class, while relational, modal, and privative families introduce implicit arguments, possible-world reference, or head-extension negation that make head-noun neighbours actively misleading.
Privative and modal accordingly remain the hardest families (below 0.08 and 0.21 respectively) across all conditions.

B1 shows the largest intersective gain over B0 ($+$0.071 R@1) but substantially hurts relational recall ($-$0.162); no fine-tuned condition recovers relational performance.
This follows directly from the supervision signal, as synonym and definition pairs encode extensional equivalence, improving families whose $C_\tau$ is extensional (intersective, subsective) while degrading or leaving unchanged those with relational or intensional $C_\tau$.
We call this the \textbf{supervision--composition matching principle}: a concept embedding benefits from fine-tuning only when the supervision encodes the same semantic composition structure as the target concept class.
Closing the gap for relational and intensional families requires type-matched supervision, which remains an open problem.

\begin{table}[t]
\centering
\small
\setlength{\tabcolsep}{5pt}
\begin{tabular}{lrrrr}
\toprule
Modifier & \textbf{B0} & \textbf{B1} & \textbf{B3} & \textbf{M2} \\
\midrule
Intersective & 0.332 & \textbf{0.403} & 0.377 & 0.364 \\
Subsective   & 0.229 & 0.279 & 0.275 & \textbf{0.284} \\
Relational   & \textbf{0.356} & 0.194 & 0.271 & 0.181 \\
Modal        & 0.194 & 0.196 & 0.188 & \textbf{0.201} \\
Privative    & 0.065 & 0.070 & \textbf{0.077} & \textbf{0.077} \\
\midrule
Overall      & 0.235 & 0.234 & \textbf{0.240} & 0.226 \\
\bottomrule
\end{tabular}
\caption{NP paraphrase R@1 by modifier family ($n$=4,000 pool).
Bold: best per row. ROC-AUC is uniformly high (0.945--0.999) and omitted.}
\label{tab:modifiers}
\end{table}


\begin{table}[t]
\centering
\small
\setlength{\tabcolsep}{6pt}
\begin{tabular}{lrrr}
\toprule
Model & R@1 & R@10 & MRR \\
\midrule
B0 (frozen)     & 0.220 & 0.608 & 0.349 \\
B1              & \textbf{0.257} & 0.649 & 0.393 \\
M2              & 0.253 & \textbf{0.665} & \textbf{0.397} \\
\midrule
\textit{DBpedia 300-step} & \textit{0.376} & \textit{0.845} & \textit{0.550} \\
\bottomrule
\end{tabular}
\caption{Cross-domain transfer to DBpedia (245 queries, 2,849 candidates). Non-italic = WordNet/Wiktionary-only training; italic = 300 in-domain DBpedia steps (upper bound). Bold: best cross-domain.}
\label{tab:transfer}
\end{table}

The aggregate results in Table~\ref{tab:modifiers} (mechanistic breakdown in the Appendix~\ref{app:error}) confirm that extensional supervision improves intersective and subsective families ($+$0.064 and $+$0.041 R@1) while degrading relational ($-$0.162, 311 regressions) and leaving modal and privative largely unchanged, consistent with P4 and the geometric account in Section~\ref{ssec:semantic-composition}.

\subsection{Cross-Domain Transfer: P1 and P4 Generalise Beyond the Training Domain}

Table~\ref{tab:transfer} confirms both principles in a new domain.
Generic fine-tuning (B1, trained on WordNet/Wiktionary) transfers partially. DBpedia R@10 improves from 0.608 to 0.649, consistent with P1---the recalibrated geometry clusters concept-equivalent expressions more tightly regardless of domain.
But 300 steps of DBpedia-specific supervision reaches R@10 = 0.845, a further $+$0.196 gain that generic fine-tuning cannot close.
This gap is the cross-domain expression of P4, showing that WordNet/Wiktionary encodes lexicographic equivalence between synonyms and dictionary definitions, while DBpedia requires mapping natural-language queries to terse property labels (\texttt{deathPlace}, \texttt{populationTotal}) whose structure never appears in lexicographic training. We validate that P1 generalises across encoder architectures in Appendix~\ref{app:backbone}. Contrastive fine-tuning improves every tested backbone, with larger gains where the frozen baseline is weaker.


\section{Conclusion}

In this paper we investigate and identify four principles that govern concept-equivalent retrieval using sentence encoders, namely the training signal, the readout strategy, the use of hard negatives, and the match between supervision and semantic composition type.
These principles offer concrete guidance to practitioners: use a sentence-fine-tuned encoder with mean pooling, add hard negatives only when calibrated scoring matters more than ranking, and match supervision to the semantic composition structure of the target concept class.

\section*{Limitations}

\textbf{Sentence-encoder backbones.} All H1/H2 ablations use \texttt{all-mpnet-base-v2}, which is already sentence-level contrastively trained; the H2 null result may therefore be specific to already-fine-tuned backbones.
On a raw pretrained transformer, cross-layer readout may still recover meaningful signal.
\textbf{Single seed.} All conditions use one training seed.
H1 margins are large enough to be robust; H2 margins ($<$1\% absolute)
may not be stable across seeds.
\textbf{Type-matched supervision.}
The supervision--composition matching principle (Section~\ref{sec:results}) indicates the need for relation-typed and negation-aware supervision pairs, which do not currently exist at scale, so relational, modal, and privative concept types remain systematically underserved.
Constructing such resources may be a valuable next step to extend this work.


\bibliography{custom}

@article{frege1892,
  title={{\"U}ber Sinn und Bedeutung},
  author={Frege, Gottlob},
  journal={Zeitschrift f{\"u}r Philosophie und philosophische Kritik},
  volume={100},
  pages={25--50},
  year={1892}
}

@inproceedings{reimers2019sentencebert,
  title={Sentence-{BERT}: Sentence Embeddings using {S}iamese {BERT}-Networks},
  author={Reimers, Nils and Gurevych, Iryna},
  booktitle={Proceedings of EMNLP-IJCNLP 2019},
  year={2019}
}

@inproceedings{gao2021simcse,
  title={{SimCSE}: Simple Contrastive Learning of Sentence Embeddings},
  author={Gao, Tianyu and Yao, Xingcheng and Chen, Danqi},
  booktitle={Proceedings of EMNLP 2021},
  year={2021}
}

@inproceedings{sung2020biosyn,
  title={Biomedical Entity Representations with Synonym Marginalization},
  author={Sung, Mujeen and Jeon, Hwisang and Lee, Jinhyuk and Kang, Jaewoo},
  booktitle={Proceedings of ACL 2020},
  year={2020}
}

@inproceedings{liu2021sapbert,
  title={Self-Alignment Pretraining for Biomedical Entity Representations},
  author={Liu, Fangyu and Shareghi, Ehsan and Meng, Zaiqiao and Basaldella, Marco and Collier, Nigel},
  booktitle={Proceedings of NAACL 2021},
  year={2021}
}

@inproceedings{remy2022biolord,
  title={{BioLORD}: Learning Ontological Representations from Definitions for Biomedical Concepts},
  author={Remy, Franck and Demuynck, Kris and Demeester, Thomas},
  booktitle={Findings of EMNLP 2022},
  year={2022}
}

@inproceedings{jawahar2019bert,
  title={What Does {BERT} Learn about the Structure of Language?},
  author={Jawahar, Ganesh and Sagot, Beno{\^i}t and Seddah, Djam{\'e}},
  booktitle={Proceedings of ACL 2019},
  year={2019}
}

@inproceedings{carvalho2025montague,
  title={Montague Semantics and Modifier Consistency Measurement in Neural Language Models},
  author={Carvalho, Danilo S. and Manino, Edoardo and Rozanova, Julia and Cordeiro, Lucas and Freitas, Andr{\'e}},
  booktitle={Proceedings of COLING 2025},
  year={2025}
}

@inproceedings{ethayarajh2019geometry,
  title={How Contextual are Contextualized Word Representations?},
  author={Ethayarajh, Kawin},
  booktitle={Proceedings of EMNLP 2019},
  year={2019}
}

@inproceedings{nickel2017poincare,
  title={Poincar{\'e} Embeddings for Learning Hierarchical Representations},
  author={Nickel, Maximilian and Kiela, Douwe},
  booktitle={Advances in Neural Information Processing Systems 30},
  year={2017}
}

@inproceedings{ettinger2018composition,
  title={Assessing Composition in Sentence Vector Representations},
  author={Ettinger, Allyson and Elgohary, Ahmed and Phillips, Colin and Resnik, Philip},
  booktitle={Proceedings of COLING 2018},
  year={2018}
}

@inproceedings{tutubalina2020fair,
  title={Fair Evaluation in Concept Normalization: a Large-Scale Comparative Analysis for {BERT}-based Models},
  author={Tutubalina, Elena and Kadurin, Artur and Miftahutdinov, Zulfat},
  booktitle={Proceedings of COLING 2020},
  year={2020}
}

@inproceedings{shwartz2019nouncompounds,
  title={A Systematic Comparison of {E}nglish Noun Compound Representations},
  author={Shwartz, Vered},
  booktitle={Proceedings of the MWE-WN Workshop, ACL 2019},
  year={2019}
}

@article{hill2016dictionary,
  title={Learning to Understand Phrases by Embedding the Dictionary},
  author={Hill, Felix and Cho, KyungHyun and Korhonen, Anna and Bengio, Yoshua},
  journal={Transactions of the Association for Computational Linguistics},
  volume={4},
  pages={17--30},
  year={2016}
}

@inproceedings{carvalho2023definitions,
  title={Learning Disentangled Representations for Natural Language Definitions},
  author={Carvalho, Danilo S. and Mercatali, Giangiacomo and Zhang, Yingji and Freitas, Andr{\'e}},
  booktitle={Findings of EACL 2023},
  year={2023}
}

@inproceedings{tenney2019bert,
  title={{BERT} Rediscovers the Classical {NLP} Pipeline},
  author={Tenney, Ian and Das, Dipanjan and Pavlick, Ellie},
  booktitle={Proceedings of ACL 2019},
  year={2019}
}

@inproceedings{peters2018elmo,
  title={Deep Contextualized Word Representations},
  author={Peters, Matthew E. and Neumann, Mark and Iyyer, Mohit and Gardner, Matt and Clark, Christopher and Lee, Kenton and Zettlemoyer, Luke},
  booktitle={Proceedings of NAACL 2018},
  year={2018}
}

@article{rogers2020bertology,
  title={A Primer in {BERT}ology: What We Know About How {BERT} Works},
  author={Rogers, Anna and Kovaleva, Olga and Rumshisky, Anna},
  journal={Transactions of the Association for Computational Linguistics},
  volume={8},
  pages={842--866},
  year={2020}
}

@inproceedings{howard2018ulmfit,
  title={Universal Language Model Fine-Tuning for Text Classification},
  author={Howard, Jeremy and Ruder, Sebastian},
  booktitle={Proceedings of ACL 2018},
  year={2018}
}

@inproceedings{valentino2024hyperbolic,
  title={Multi-Relational Hyperbolic Word Embeddings from Natural Language Definitions},
  author={Valentino, Marco and Carvalho, Danilo and Freitas, Andr{\'e}},
  booktitle={Proceedings of EACL 2024},
  year={2024}
}

@article{mitchell2010composition,
  title={Composition in Distributional Models of Semantics},
  author={Mitchell, Jeff and Lapata, Mirella},
  journal={Cognitive Science},
  volume={34},
  number={8},
  pages={1388--1429},
  year={2010}
}

@inproceedings{baroni2010nouns,
  title={Nouns are Vectors, Adjectives are Functions: Experiments with Compositional Models of Meaning},
  author={Baroni, Marco and Zamparelli, Roberto},
  booktitle={Proceedings of EMNLP 2010},
  pages={1183--1193},
  year={2010}
}

@incollection{partee1995lexical,
  title={Lexical Semantics and Compositionality},
  author={Partee, Barbara H.},
  booktitle={An Invitation to Cognitive Science: Language},
  editor={Gleitman, Lila and Liberman, Mark},
  volume={1},
  pages={311--360},
  year={1995},
  publisher={MIT Press},
  address={Cambridge, MA}
}

\clearpage
\appendix

\section{Reproducibility}
\label{app:reproducibility}

All backbone models used in this work (\texttt{all-mpnet-base-v2}, \texttt{paraphrase-mpnet-base-v2}, \texttt{e5-base-v2}, \texttt{jina-embeddings-v5-text-nano}, and \texttt{paraphrase-MiniLM-L6}) are publicly available from the Sentence-Transformers library \citep{reimers2019sentencebert}.
Training data are drawn exclusively from WordNet and Wiktionary, both publicly available; the exact pair-generation procedure is described in the Method section and Appendix~\ref{app:data_examples}.
Full hyperparameters (learning rate, batch size, warmup, layerwise decay schedule, pooling window $K$) are listed in Appendix~\ref{app:hparams}.
The two evaluation benchmarks introduced in this work---the DBpedia semantic-gap benchmark and the modifier-labeled NP paraphrase suite---are released alongside the paper.
Training code, configuration files, and model checkpoints for all eight conditions are released publicly.\footnote{Code release is withheld during double-blind review to preserve author anonymity.}

\section{Training Data Examples}
\label{app:data_examples}

Table~\ref{tab:data_examples} illustrates the three pair types used during training.
Synonym pairs align near-synonymous terms from the same source.
Term-to-definition (t2d) pairs align a term with its dictionary definition; the headword is masked from the definition before training to prevent lexical shortcutting.
Cross-source definition--definition (d2d) pairs are generated on-the-fly by pairing definitions of the same concept across WordNet and Wiktionary.

\begin{table*}[h]
\centering
\small
\setlength{\tabcolsep}{5pt}
\begin{tabular}{llp{5.5cm}p{5.5cm}}
\toprule
Type & Source & Left & Right \\
\midrule
\multirow{3}{*}{Synonym}
  & WN  & fast          & quick \\
  & WN  & joyful        & elated \\
  & Wkt & physician     & doctor \\
\midrule
\multirow{3}{*}{t2d}
  & WN  & democracy     & a political system in which the supreme power lies in a body of citizens who can elect people to represent them \\
  & WN  & photosynthesis & the process in green plants by which carbohydrates are synthesised from carbon dioxide and water using light as an energy source \\
  & Wkt & pandemic      & a disease that affects a very large number of people throughout the world \\
\midrule
\multirow{2}{*}{d2d}
  & WN$\leftrightarrow$Wkt & a hot drink made by infusing the dried crushed leaves of the tea plant in boiling water & an aromatic beverage prepared by steeping cured leaves of the tea plant \\
  & WN$\leftrightarrow$Wkt & a state of complete mental and physical well-being & the state of being comfortable, healthy, or happy \\
\bottomrule
\end{tabular}
\caption{Representative training pair examples for each supervision type. WN = WordNet; Wkt = Wiktionary. Headwords are masked from definitions during training.}
\label{tab:data_examples}
\end{table*}


\section{Modifier Composition Types}
\label{app:modifier_types}

The five modifier families used throughout this paper originate in formal semantics \citep{partee1995lexical,carvalho2025montague}.
Table~\ref{tab:modifier_types} provides an accessible reference with concrete adjective--noun examples and plain-language descriptions of the underlying meaning mechanism, for readers less familiar with this typology.

\definecolor{intcol}{RGB}{209,231,221}
\definecolor{subcol}{RGB}{207,226,255}
\definecolor{relcol}{RGB}{255,243,205}
\definecolor{modcol}{RGB}{230,215,255}
\definecolor{privcol}{RGB}{255,218,218}

\begin{table*}[h]
\centering
\small
\setlength{\tabcolsep}{6pt}
\begin{tabular}{llp{3.5cm}p{6cm}}
\toprule
\textbf{Type} & \textbf{Example} & \textbf{Paraphrase} & \textbf{Meaning mechanism} \\
\midrule
\rowcolor{intcol}  Intersective & \emph{red car}       & \emph{a car that is red}                        & Intersection of sets: the referent is in both the modifier and head extension. \\
\rowcolor{subcol}  Subsective   & \emph{large ant}     & \emph{an ant that is large for an ant}           & Restriction within the head extension; the modifier is relative to the noun class, not absolute. \\
\rowcolor{relcol}  Relational   & \emph{electric car}  & \emph{a car powered by electricity}              & Introduces an implicit relation argument absent from the surface string; meaning depends on context. \\
\rowcolor{modcol}  Modal        & \emph{alleged thief} & \emph{someone alleged to be a thief}             & References possible worlds; the head denotation is not entailed by the compound. \\
\rowcolor{privcol} Privative    & \emph{toy gun}       & \emph{something that looks like a gun but is not a gun} & Negates the head extension: the referent is explicitly excluded from the head class. \\
\bottomrule
\end{tabular}
\caption{The five modifier composition families with adjective--noun examples and plain-language meaning descriptions. Each family requires a structurally different semantic operator $C_\tau$ (Section~\ref{sec:prelim}), which is why a single latent geometry cannot support all of them equally well under uniform supervision.}
\label{tab:modifier_types}
\end{table*}


\section{Theoretical Foundations of Conceptual Compositionality}
\label{app:theory}

This appendix formalizes the theoretical perspective used in the main text.
It states the structural assumptions under which the observed modifier-family differences, pooling null result, and fine-tuning effects follow naturally.

\subsection{Setup and Notation}

Let $(\mathcal{X}, \circ_\tau)$ denote a typed expression algebra, where $\tau$ indexes composition types such as intersective, subsective, relational, modal, and privative modification.
Let $\sem{\cdot}: \mathcal{X} \to \mathcal{C}$ be a semantic interpretation function into a concept domain $\mathcal{C}$.
Let $f: \mathcal{X} \to \mathbb{R}^d$ be a sentence encoder, and let $s: \mathbb{R}^d \times \mathbb{R}^d \to \mathbb{R}$ be the similarity score used for retrieval.
We write $x \sim y$ when $\sem{x} = \sem{y}$.

We use the $\varepsilon_\tau$-compositionality definition from \S\ref{sec:prelim}: encoder $f$ is $\varepsilon_\tau$-compositional for type $\tau$ if there exists $\Phi_\tau : \mathbb{R}^d \times \mathbb{R}^d \to \mathbb{R}^d$ such that $\| f(m \circ_\tau h) - \Phi_\tau(f(m), f(h)) \| \leq \varepsilon_\tau$ for all valid modifier--head pairs $(m,h)$ of type $\tau$.
The quantity $\varepsilon_\tau$ measures the distortion incurred when a typed semantic operator is realized in a single geometric representation space.
Small $\varepsilon_\tau$ indicates that the encoder supports the corresponding composition family; large $\varepsilon_\tau$ indicates an operator--geometry mismatch.

\subsection{Low-Distortion Homomorphism Requirement}

\begin{theorem}[Approximate Homomorphism Requirement]
Suppose retrieval over a composition family $\tau$ is stable under denotational equivalence: for any $x = m \circ_\tau h$ and any $y$ such that $x \sim y$, the representations $f(x)$ and $f(y)$ lie within radius $\delta$ of one another. Then there exists a latent operator $\Phi_\tau$ whose empirical distortion over observed compositions is at most the optimal prediction error from constituent embeddings:
\[
\sup_{(m,h)} \left\| f(m \circ_\tau h) - \Phi_\tau(f(m), f(h)) \right\| \leq \varepsilon_\tau^*,
\]
where $\varepsilon_\tau^*$ is the minimum achievable reconstruction error over operators on $\mathbb{R}^d \times \mathbb{R}^d$.
\end{theorem}

\begin{proof}
Let $Z_\tau = \{(f(m), f(h), f(m \circ_\tau h))\}$ be the set of observed constituent--composition triples.
Define $\Phi_\tau$ as any minimizer of worst-case reconstruction error over $Z_\tau$; existence on finite support is immediate by table-lookup construction, assigning each observed input pair to a centroid of the corresponding target set.
It remains to bound $\varepsilon_\tau^*$ using the stability assumption.
For each pair $(f(m), f(h))$, the target $f(m \circ_\tau h)$ may vary across surface realizations of the same composition.
Stability under equivalence ensures all such realizations lie within a $\delta$-ball of one another: $\|f(x) - f(y)\| \leq \delta$ whenever $\sem{x} = \sem{y}$.
Any operator that maps each input pair to a representative point of the corresponding $\delta$-ball therefore achieves worst-case reconstruction error at most $\delta$ over the observed support.
Hence $\varepsilon_\tau^* \leq \delta$, establishing the bound.
Conversely, if no low-distortion $\Phi_\tau$ exists ($\varepsilon_\tau^* \gg \delta$), then the composed representations are not determined by their constituents, and expressions with identical semantic structure but different surface forms scatter to unrelated regions---contradicting retrieval stability.
\end{proof}

\paragraph{Interpretation.}
The theorem makes precise why compositional retrieval is stronger than ordinary sentence similarity: success requires not only clustering equivalent expressions, but also realizing a predictable latent operator for the relevant composition family.

\subsection{Identifiability from Equivalence Supervision}
\label{app:quotient-geometry}

\begin{theorem}[Supervision Identifiability Boundary]
If training supervision consists only of equivalence constraints of the form $x \sim y \Rightarrow f(x) \approx f(y)$, then latent operators $\Phi_\tau$ are identifiable only up to equivalence-preserving transformations. In particular, relation-typed and intensional parameters not expressed in the equivalence labels are underdetermined.
\end{theorem}

\begin{proof}
Let $[x]_\sim = \{y \in \mathcal{X} : \sem{y} = \sem{x}\}$ denote the denotational equivalence class of $x$.
Equivalence supervision constrains $f$ only through labeled pairs: it enforces $f(x) \approx f(y)$ when $x \sim y$, and $f(x) \not\approx f(z)$ when $x \not\sim z$.
This determines the quotient geometry of $f$ over $\sim$---the relative placement of equivalence class centroids---but leaves the internal compositional structure of $f$ underdetermined.

Formally, let $T : \mathbb{R}^d \to \mathbb{R}^d$ be any invertible map satisfying: $T(v)$ lies within the $\delta$-ball of $[x]_\sim$ whenever $v = f(x)$.
Such transformations form a non-trivial family; for example, any orthogonal rotation within each class neighborhood qualifies.
The transformed encoder $\tilde{f} = T \circ f$ satisfies all equivalence constraints, since $\|\tilde{f}(x) - \tilde{f}(y)\| = \|Tf(x) - Tf(y)\| \leq \delta$ whenever $f(x)$ and $f(y)$ already lie within $\delta$ of one another.
For any operator $\Phi_\tau$ compatible with $f$, the conjugated operator
\[
\tilde{\Phi}_\tau = T \circ \Phi_\tau \circ (T^{-1} \times T^{-1})
\]
satisfies $\tilde{\Phi}_\tau(\tilde{f}(m), \tilde{f}(h)) = T(\Phi_\tau(f(m), f(h)))$, matching $\tilde{f}(m \circ_\tau h)$ whenever $\Phi_\tau$ matched $f(m \circ_\tau h)$.
Since $T$ can be chosen to permute the internal geometry arbitrarily within equivalence neighborhoods, the composition operator is identifiable only up to the family of such $T$.
In particular, whether a match arises from intersection, relation binding, modal evaluation, or privative exclusion cannot be inferred from equivalence labels alone, unless those distinctions are explicitly represented in the training signal.
\end{proof}

\paragraph{Interpretation.}
This result formalizes the supervision--composition matching principle: synonym and definition pairs identify quotient geometry over denotational equivalence, but they do not identify all typed semantic operators.

\subsection{Head-Preservation Bias}

\begin{theorem}[Geometric Head Bias]
Consider a retrieval geometry in which composed noun phrases are encouraged to remain close to their heads: $s(f(m \circ_\tau h), f(h))$ is high relative to unrelated heads. Then composition families whose denotations preserve head membership admit lower distortion under this geometry than composition families requiring head exclusion or non-local relation binding.
\end{theorem}

\begin{proof}[Proof sketch]
The argument proceeds by examining the semantic relationship between composition and head extension for each family type.
For intersective modification, $\sem{m \circ h} = \sem{m} \cap \sem{h} \subseteq \sem{h}$: every instance of the composition is also an instance of the head.
Geometric proximity $s(f(m \circ h), f(h)) \geq \alpha$ is therefore semantically coherent---the representation should lie in the same region as the head.
For subsective modification, the composition is head-preserving in extension ($\sem{m \circ h} \subseteq \sem{h}$), so the same argument applies.
For relational modification, $\sem{m \circ h}$ depends on an implicit relation argument $R$ not encoded in either $f(m)$ or $f(h)$ alone; head proximity provides no information about $R$, so the prediction from constituent embeddings incurs additional residual.
For modal modification, the composition introduces an evidential variable indicating uncertainty about head membership; proximity to $f(h)$ conflates semantically certain and uncertain instances.
For privative modification, $\sem{m \circ_{\mathrm{priv}} h} \cap \sem{h} = \emptyset$: proximity to $f(h)$ is maximally misleading.
The distortion $\varepsilon_\tau$ therefore increases monotonically across the hierarchy intersective, subsective, relational, modal, privative, following the operator--geometry alignment.
\end{proof}

\paragraph{Interpretation.}
This theorem explains why the modifier hierarchy in the experiments is not merely empirical: the ranking follows from how well each operator family aligns with head-centered similarity.

\subsection{Privative Incompatibility under a Single Global Metric}

\begin{theorem}[Privative Metric Conflict]
Assume the encoder is consistent with respect to head instances: for all genuine instances $y \in \sem{h}$, $s(f(y), f(h)) \geq \alpha - \eta$ for small $\eta > 0$ (instances cluster near their category head).
Then no single global metric can simultaneously satisfy, for all privative constructions $m \circ_{\mathrm{priv}} h$, the following three constraints:
(i) $s(f(m \circ_{\mathrm{priv}} h),\, f(h)) \geq \alpha$ (high proximity to lexical head);
(ii) $\sem{m \circ_{\mathrm{priv}} h} \cap \sem{h} = \emptyset$ (semantic exclusion from head extension);
(iii) $s(f(m \circ_{\mathrm{priv}} h),\, f(y)) < \beta$ for all $y \in \sem{h}$, with $\beta < \alpha - \eta$ (retrieval isolation from true head instances).
\end{theorem}

\begin{proof}
Suppose (i) and the consistency assumption both hold.
Let $y \in \sem{h}$ be any genuine head instance.
By the consistency assumption, $f(y)$ lies within distance $\eta$ (in the metric induced by $s$) of $f(h)$.
By the triangle inequality, proximity of $f(m \circ_{\mathrm{priv}} h)$ to $f(h)$ implies proximity to $f(y)$:
\[
s\!\left(f(m \circ_{\mathrm{priv}} h),\, f(y)\right) \geq \alpha - \eta - \epsilon,
\]
where $\epsilon > 0$ is determined by the metric geometry.
For $\eta$ and $\epsilon$ small relative to the gap $\alpha - \beta > 0$, this contradicts constraint (iii), which requires $s(f(m \circ_{\mathrm{priv}} h), f(y)) < \beta < \alpha - \eta$.
Hence (i) and (iii) cannot hold simultaneously.
Relaxing (i) to allow $s(f(m \circ_{\mathrm{priv}} h), f(h)) < \alpha$ sacrifices lexical anchoring.
Since the same global score $s$ governs both, no uniform choice satisfies all three constraints without additional type-discriminative structure.
\end{proof}

\paragraph{Interpretation.}
Privative failures are therefore not just cases of insufficient training. They expose a representational conflict in ordinary embedding spaces: the model must be close to the head in one sense and far from it in another.

\subsection{Equivalence Training Reorganizes Geometry}

\begin{theorem}[Capacity-Reorganization Principle]
Let contrastive equivalence training update an encoder primarily through pairwise attraction of positives and repulsion of negatives in a fixed $d$-dimensional embedding space. Then, absent architectural expansion or new representational channels, improvements in retrieval arise from reorganization of the existing geometry rather than from increased representational capacity.
\end{theorem}

\begin{proof}
The encoder output dimension remains $d$ throughout fine-tuning, so the ambient embedding space $\mathbb{R}^d$ and its capacity are fixed.
Contrastive updates are gradient steps on pairwise cosine similarities: positive pairs receive attraction gradients and negative pairs receive repulsion gradients, modifying the weight matrices of the transformer but not adding independent output coordinates.
The effect on the embedding geometry is a change in the covariance structure of $\{f(x)\}_{x \in \mathcal{X}}$: specifically, the singular value distribution of the embedding matrix shifts.
Anisotropy (the fraction of variance captured by the leading singular direction) decreases as attraction and repulsion spread the embedding cloud more uniformly across directions.
Crucially, the effective rank (number of dimensions carrying non-negligible variance, measured for example by the participation ratio $(\sum_i \lambda_i)^2 / \sum_i \lambda_i^2$ over singular values $\lambda_i$) may remain approximately stable even as the distribution across directions changes.
Contrastive training therefore reorganizes the geometry (changing which directions carry semantic signal) without expanding it (the number of independent representational axes is bounded by $d$ before and after).
\end{proof}

\paragraph{Interpretation.}
The theorem supports the claim that concept-equivalence fine-tuning recalibrates a pre-existing semantic space rather than expanding it.

\subsection{Distortion and Operator Complexity}

\begin{theorem}[Operator Distortion Lower Bound]
Let $\kappa(\tau)$ denote the number of independent semantic degrees of freedom required to realize composition type $\tau$, including latent relation variables, modal parameters, and exclusion constraints. If these degrees of freedom are not encoded in the supervision or representation, then any realization of $\tau$ in a $d$-dimensional embedding space incurs nonzero distortion. Moreover, distortion must increase as the unsupervised degrees of freedom of $\tau$ increase relative to the information preserved by $f$.
\end{theorem}

\begin{proof}[Proof sketch]
An encoder can faithfully realize $C_\tau$ only to the extent that the $\kappa(\tau)$ semantic variables determining that operator are preserved in the representation.
If two expressions differ in a latent variable required by $C_\tau$ (a relation argument, an intensional parameter, or a privative exclusion flag) but the encoder maps them to near-identical vectors, then no downstream $\Phi_\tau$ can recover the distinction.
Each such collapsed variable represents at least one dimension of the target space that is rendered inaccessible.
By a counting argument over the required variable assignments: if $k$ independent binary distinctions collapse (each mapped to indistinguishable vectors), then at least $k$ semantic contrasts cannot be expressed at the output, and the resulting retrieval errors are bounded away from zero for each.
Since the number of collapsed variables is bounded below by $\kappa(\tau)$ minus the number of variables identified by the training signal, distortion increases monotonically as $\kappa(\tau)$ grows relative to supervised information.
Making this argument tight requires an information-theoretic lower bound on the minimum encoding error for a $d$-dimensional representation with partially-observed supervision, which we defer to future work.
\end{proof}

\paragraph{Interpretation.}
This result gives a formal reason why relational, modal, and privative families remain difficult even when extensional concept equivalence improves.

\subsection{Readout Invariance under Final-Layer Concentration}

\begin{theorem}[Readout Invariance]
Let $h_l(x) \in \mathbb{R}^d$ denote the output of transformer layer $l$ for expression $x$, and let $f_r(x) = \sum_{l=1}^L w_l\, h_l(x)$ be the pooled representation under readout $r = \{w_l\}_{l=1}^L$ with $\sum_l w_l = 1$.
Suppose semantic content concentrates in the final layer: $h_l(x) = h_L(x) + \xi_l(x)$ where $\|\xi_l(x)\| \leq \gamma$ for all $l < L$ and all $x$.
Then for any two readout strategies $r$ and $r'$,
\[
\| f_r(x) - f_{r'}(x) \| \leq 2\gamma.
\]
In particular, when $\gamma$ is small, the choice of readout has negligible effect on retrieval.
\end{theorem}

\begin{proof}
\begin{align*}
\| f_r(x) - f_{r'}(x) \|
&= \left\| \sum_l (w_l - w'_l)\, h_l(x) \right\| \\
&= \left\| \sum_l (w_l - w'_l)\bigl(h_L(x) + \xi_l(x)\bigr) \right\| \\
&= \left\| h_L(x)\underbrace{\sum_l (w_l - w'_l)}_{=\,0}\right.\\
&\qquad\left. {}+ \sum_l (w_l - w'_l)\,\xi_l(x) \right\| \\
&\leq \sum_l |w_l - w'_l| \cdot \|\xi_l(x)\| \;\leq\; 2\gamma,
\end{align*}
where $\sum_l (w_l - w'_l) = 0$ because both weight vectors sum to 1, and $\sum_l |w_l - w'_l| \leq 2$ because the $L_1$ distance between two probability vectors is at most 2.
\end{proof}

\paragraph{Interpretation.}
This theorem formalizes the P2 finding that pooling strategy is a null factor in fine-tuned encoders.
Equivalence training concentrates semantic content in the final transformer layer, driving $\gamma$ toward zero: intermediate layers carry only low-variance residuals, so the choice among mean pooling, CLS, weighted mean, or gated readout produces embeddings within $2\gamma$ of one another.
Pre-fine-tuning, layer representations are more heterogeneous ($\gamma$ larger) and readout strategies diverge correspondingly.
The theorem also predicts that readout sensitivity should correlate with anisotropy: both reflect the distribution of semantic content across encoder layers and directions.

\subsection{Summary}

The theoretical picture is that sentence encoders approximate a homomorphic mapping from a typed semantic algebra into a fixed geometric space.
The seven theorems above provide formal grounding for the four empirical principles of the main paper.
\textbf{P1} (recalibration, not expansion) is supported by the Capacity-Reorganization Principle (Theorem~5). Contrastive fine-tuning changes the covariance structure of the space but not its dimension, so improvements arise from redistribution rather than augmentation.
\textbf{P2} (final-layer concentration) is supported by the Readout Invariance theorem. Once semantic content concentrates in the final layer, all linear pooling strategies converge to within $2\gamma$ of one another.
\textbf{P3} (calibration and ranking separable via hard negatives) is supported by the Approximate Homomorphism and Head-Preservation theorems (Theorems~1, 3). Stable retrieval requires a low-distortion composition operator, and head-preserving families achieve this under the existing geometry while operator-complex families require explicit negative supervision.
\textbf{P4} (supervision must match composition type) is supported by the Supervision Identifiability theorem (Theorem~2). Equivalence labels identify geometry only up to equivalence-preserving transformations; composition-type-specific operators are underdetermined without targeted supervision.
The Privative Metric Conflict (Theorem~4) and Operator Distortion Lower Bound (Theorem~6) additionally explain why privative, modal, and relational families remain hard even after equivalence fine-tuning. They require representational structure that a single global metric over synonym-definition pairs cannot provide.

\section{Training Hyperparameters}
\label{app:hparams}

Table~\ref{tab:hparams} lists the full set of hyperparameters used across all conditions.

\begin{table}[t]
\centering
\small
\setlength{\tabcolsep}{6pt}
\begin{tabular}{lr}
\toprule
Hyperparameter & Value \\
\midrule
Optimizer          & AdamW \\
Base learning rate & $2{\times}10^{-5}$ \\
LR schedule        & Linear warmup + decay \\
Warmup steps       & 100 \\
Layerwise LR decay & 0.90 per layer \\
Pooling LR multiplier & $5{\times}$ base \\
Weight decay       & 0.01 \\
Gradient clip norm & 1.0 \\
Batch size         & 16 \\
Training steps     & 8{,}000 \\
Max sequence length & 128 tokens \\
InfoNCE temperature $\tau$ & 0.05 \\
Hard-neg margin    & 0.20 \\
Layer window $K$   & 4 \\
\bottomrule
\end{tabular}
\caption{Hyperparameters shared across all fine-tuned conditions.}
\label{tab:hparams}
\end{table}

\section{Loss Variant: Unified InfoNCE}
\label{app:loss}

A simplified variant (B1-unified) folds hard negatives directly into the InfoNCE
denominator, eliminating the BCE term. Results are shown in Table~\ref{tab:unified}.
B1-unified improves retrieval ranking (t2d R@10 $+$0.012, t2d MRR $+$0.014)
but substantially reduces calibration (pair ROC-AUC $-$0.058) and hard-negative
stress-test robustness (negate ROC-AUC $-$0.182).
The BCE term provides a fixed-magnitude gradient per hard negative regardless of
batch size; the unified formulation dilutes the hard-negative signal in proportion to
batch size, reducing its effectiveness for calibration.
The primary BCE variant is preferred when downstream use requires calibrated pair
scoring or robustness to semantic perturbations.

\begin{table*}[h]
\centering
\small
\setlength{\tabcolsep}{4pt}
\begin{tabular}{lrrrrrr}
\toprule
Condition & t2d R@10 & t2d MRR & Pair AUC & HH d2t & Negate & NP R@10 \\
\midrule
B1             & 0.654 & 0.537 & \textbf{0.963} & 0.686 & \textbf{0.980} & 0.656 \\
B1-unified     & \textbf{0.666} & \textbf{0.551} & 0.905 & \textbf{0.692} & 0.798 & 0.656 \\
\bottomrule
\end{tabular}
\caption{B1 vs.\ B1-unified across the same evaluation columns as Table~\ref{tab:h1}.
Unified InfoNCE improves retrieval ranking but substantially reduces calibration
and stress-test robustness. Bold: best per column.}
\label{tab:unified}
\end{table*}

\section{Hard-Negative Rules}
\label{app:hardnegs}

Five rule-based transforms are applied to positive definitions to generate hard negatives.
The same five rules are used during both training (BCE hard-negative term) and evaluation
(stress-test ROC-AUC), so the stress tests directly probe the discriminations the model
was trained on.
Each rule is applied once per positive definition; all generated negatives are used.

\paragraph{Antonym flip.}
A key content word in the definition is identified and replaced by its antonym using
a WordNet antonymy lookup.
This tests whether the model can distinguish conceptually opposite definitions that share
the same syntactic frame
(\emph{``a sustained increase in prices''} $\to$ \emph{``a sustained decrease in prices''}).

\paragraph{Lexical negation.}
The first finite verb or copula in the definition is negated by inserting \emph{not}
(\emph{``is a warm-blooded vertebrate''} $\to$ \emph{``is not a warm-blooded vertebrate''}).
This is the most challenging type: the negative differs from the positive by a single token,
and the models must assign lower similarity to what is linguistically a denial of the
positive concept.

\paragraph{Same-POS random swap.}
The positive definition is replaced with a randomly sampled definition from a different
word sharing the same coarse POS tag (noun, verb, adjective, adverb).
This tests broad categorical discrimination: the model must not rely solely on POS-level
similarity cues to rank candidates.

\paragraph{Same-POS prefix swap.}
As above, but the replacement definition is drawn from a word sharing the same morphological
prefix (\emph{un-}, \emph{re-}, \emph{over-}, etc.).
The prefix overlap makes this harder than random same-POS swaps, as it introduces surface
cues that can mislead lexical-similarity-based models.

\paragraph{Ontological type swap.}
The definition is replaced with one from a concept of a different ontological type,
where type is drawn from a small closed set
(person, organization, location, event, artifact, biological entity, abstract concept).
This tests whether the model has encoded the coarse ontological category of a concept
rather than merely its surface distribution
(\emph{``the capital city of France''} $\to$ \emph{``a multinational technology corporation''}).

\paragraph{Evaluation procedure.}
For each test pair (query $q$, positive definition $d^+$), a negative $d^-$ is generated
by applying one rule to $d^+$.
The model scores $\cos(f(q), f(d^+))$ and $\cos(f(q), f(d^-))$;
ROC-AUC is computed over the resulting binary discrimination across all test pairs for
that rule.
A score of 1.0 indicates perfect discrimination; 0.5 is chance.

\section{Geometry Diagnostics}
\label{app:geometry}

We measure two complementary geometry properties on the in-domain split.
\textbf{Anisotropy} is the mean cosine similarity over 10,000 random pairs of
encoded texts: lower values indicate a more isotropic, uniformly spread space.
\textbf{Effective rank} is $\exp(H(\sigma))$, the exponential Shannon entropy of the
normalised singular-value spectrum: higher values indicate that more dimensions carry
meaningful variance.
Together these characterise whether the model is using the available dimensions
efficiently and without degenerate clustering.

\begin{table}[h]
\centering
\small
\setlength{\tabcolsep}{5pt}
\begin{tabular}{llrr}
\toprule
ID & Condition & Aniso. & E.\ rank \\
\midrule
B0    & Frozen mean             & 0.126 & 247.0 \\
B0-WM & Frozen wtd.\ mean      & 0.303 & 177.0 \\
B1    & Mean + HN              & \textbf{0.012} & \textbf{247.2} \\
B3    & Weighted mean          & 0.166 & 193.9 \\
M1    & Weighted mean + HN     & 0.015 & 239.4 \\
\bottomrule
\end{tabular}
\caption{Geometry diagnostics (in-domain split).
Anisotropy: mean cosine over random pairs (lower $=$ better).
Effective rank: $\exp(H(\sigma))$, entropy of the normalised eigenspectrum (higher $=$ better).
Hard negatives are the dominant geometric factor (B1 vs.\ B3).}
\label{tab:geometry}
\end{table}

\begin{figure}[h]
\centering
\includegraphics[width=\linewidth]{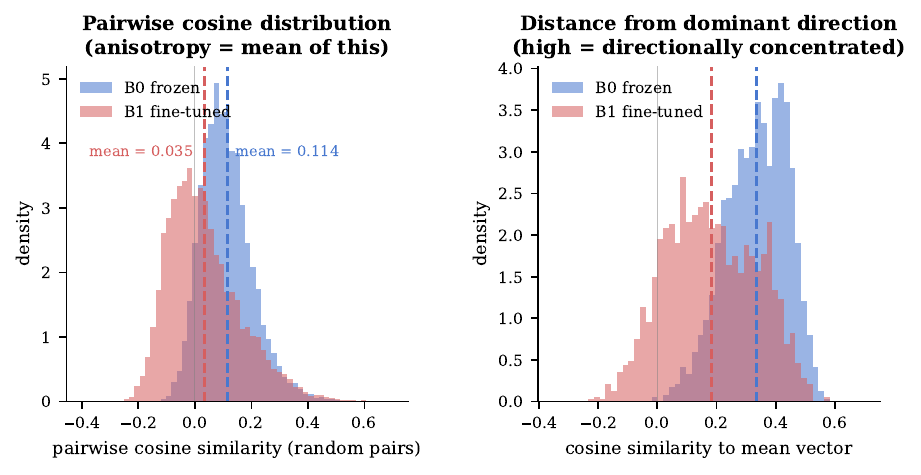}
\caption{Anisotropy and effective rank across encoder conditions. Hard-negative supervision (B1, M1) is the dominant factor driving anisotropy collapse; pooling strategy alone has a secondary effect.}
\label{fig:anisotropy_comparison}
\end{figure}

\begin{figure}[h]
\centering
\includegraphics[width=\linewidth]{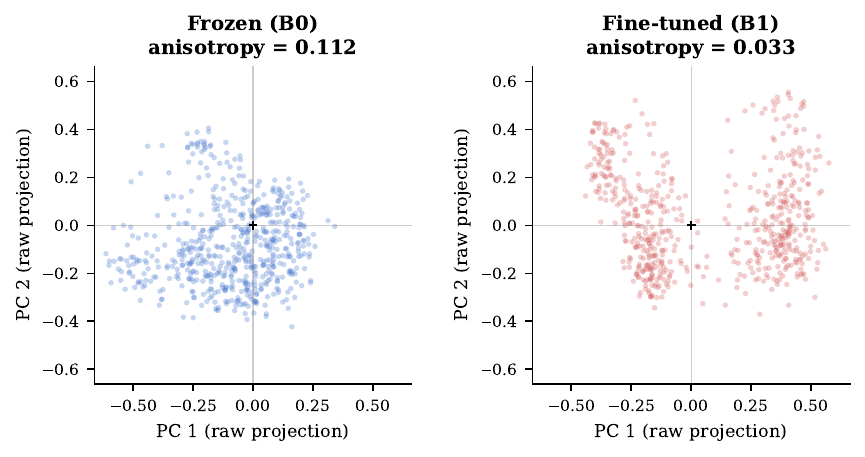}
\caption{Scatter of anisotropy vs.\ effective rank across encoder conditions. Fine-tuning with hard negatives (B1, M1) achieves low anisotropy and high effective rank simultaneously; the frozen weighted mean (B0-WM) is the geometric outlier.}
\label{fig:anisotropy_scatter}
\end{figure}

\paragraph{Why B0 already has good geometry.}
The frozen mean-pool baseline (B0) achieves surprisingly well-calibrated geometry
(anisotropy 0.126, effective rank 247.0) without any task-specific training.
This is a direct consequence of the backbone's history: \texttt{all-mpnet-base-v2}
was itself contrastively trained on 1 billion sentence pairs for sentence-level
similarity \citep{reimers2019sentencebert}, which spreads representations isotropically
and concentrates signal in the final layer.
The per-layer anisotropy plot (Figure~\ref{fig:layerwise}) makes this concrete:
layer-12 anisotropy in the frozen encoder (0.146) is far below layers 5--11 (0.45--0.50).
We inherit this structure without modification; the B0 baseline is strong not because of
anything we did, but because the backbone was already well-calibrated.

\paragraph{Why B0-WM is substantially worse.}
B0-WM (frozen uniform weighted mean over the last four layers) has anisotropy 0.303
and effective rank only 177.0---substantially worse than B0 despite using the same frozen
backbone.
The reason is straightforward: averaging layers 9--11 into the layer-12 output introduces
high-anisotropy (0.45--0.50) intermediate representations into the mixture, diluting the
well-calibrated final layer.
This is the same mechanism as layer-12 dominance, operating at the pooling level:
even without gradient flow, choosing the wrong readout is geometrically costly.
The retrieval penalty mirrors the geometry penalty---B0-WM achieves only 0.392 t2d R@10
vs.\ B0's 0.552 (Table~\ref{tab:h1}).

\paragraph{Fine-tuning reorganises, not expands.}
The most striking geometry finding is that fine-tuning on 3.3 million concept pairs
leaves effective rank essentially unchanged (B0: 247.0 $\to$ B1: 247.2) while collapsing
anisotropy (0.126 $\to$ 0.012).
Contrastive training does not recruit new dimensions; it redistributes variance across
existing ones, rotating and rescaling directions so that equivalent concepts land in the
same neighbourhood and hard negatives are pushed apart.
The representational capacity (effective rank $\approx$ 247) was already present in the frozen
encoder, contributed by the prior sentence-transformer training.
Concept-equivalence fine-tuning directs it, but cannot create it.

\paragraph{Hard negatives are the dominant geometric factor.}
The B3 vs.\ B1 comparison isolates the contribution of hard-negative supervision
(same backbone, same pooling, same budget; only the BCE hard-negative term differs):
B3 anisotropy is 0.166, effective rank 193.9; B1 achieves 0.012 and 247.2.
Hard negatives require the model to push semantically distinct but geometrically close
representations apart, which directly reduces anisotropy and expands the effective rank.
By contrast, varying the pooling architecture while holding supervision fixed
has a much smaller effect: M1 (weighted mean + HN) achieves anisotropy 0.015 and rank
239.4, only marginally below B1 (mean pool + HN) despite a different readout.
This confirms that the geometry improvement is driven by the training signal, not the
readout architecture---consistent with the H2 retrieval null result.

\section{NP Paraphrase Error Examples}
\label{app:error}

Table~\ref{tab:error} lists the three hardest B1 failures per modifier family,
sorted by rank of the correct candidate (descending).
Rank is computed over the full 4,000-item pool.
``Correct sim'' and ``Retrieved sim'' are cosine similarities to the correct
paraphrase and to the rank-1 retrieved candidate respectively.
The final column classifies the failure mode as discussed in Section~\ref{app:error}.

\begin{table*}[h]
\centering
\scriptsize
\setlength{\tabcolsep}{3pt}
\begin{tabular}{llp{3.0cm}rrp{3.0cm}rl}
\toprule
Family & Query & Correct paraphrase & Cor.\ sim & Rank & Retrieved rank-1 & Ret.\ sim & Failure mode \\
\midrule
\multirow{3}{*}{\textsc{intersective}}
  & clean plate    & free of dirt plate         & 0.487 &  83 & clean dish                           & 0.832 & \textit{adj.\ synonym} \\
  & dirty temple   & not clean temple           & 0.577 &  59 & filthy temple                        & 0.937 & \textit{adj.\ synonym} \\
  & dirty cup      & not clean cup              & 0.583 &  58 & filthy cup                           & 0.838 & \textit{adj.\ synonym} \\
\midrule
\multirow{3}{*}{\textsc{subsective}}
  & effective student  & high impact student    & 0.689 &  48 & effective teacher                    & 0.892 & \textit{head-noun confusion} \\
  & effective trainer  & high impact trainer    & 0.726 &  42 & highly effective trainer             & 0.933 & \textit{degree-modifier variant} \\
  & effective manager  & high impact manager    & 0.739 &  30 & highly effective manager             & 0.918 & \textit{degree-modifier variant} \\
\midrule
\multirow{3}{*}{\textsc{relational}}
  & community cup  & container owned by local communities & 0.253 & 864 & city cup                   & 0.899 & \textit{compound cross-retrieval} \\
  & city cup       & container associated with municipal service & 0.264 & 637 & community cup         & 0.899 & \textit{compound cross-retrieval} \\
  & summer account & profile related to the warm season & 0.335 & 580 & student account            & 0.771 & \textit{compound cross-retrieval} \\
\midrule
\textsc{modal}
  & alleged culprit    & not proven culprit     & 0.633 &  94 & purported culprit$^\dagger$          & 0.929 & \textit{annotation limit} \\
\midrule
\multirow{3}{*}{\textsc{privative}}
  & toy ticket         & pass lookalike         & 0.347 & 497 & play ticket$^\ddagger$               & 0.807 & \textit{annotation limit} \\
  & fraudulent account & profile lookalike      & 0.349 & 470 & forged account                       & 0.912 & \textit{adj.\ synonym} \\
  & spurious ticket    & pass lookalike         & 0.372 & 437 & bogus ticket                         & 0.927 & \textit{adj.\ synonym} \\
\bottomrule
\end{tabular}
\caption{Hardest B1 failures per modifier family (up to 3; 4,000-item pool). Rank = position of correct paraphrase; lower = harder failure.
Failure modes: \textit{adj.\ synonym} = same-adjective surface variant ranked above definitional paraphrase; \textit{head-noun confusion} = correct modifier, wrong noun; \textit{degree-modifier variant} = degree form ranked above definition; \textit{compound cross-retrieval} = adjacent compound preferred over relational-clause paraphrase; \textit{annotation limit} = retrieved item is a valid paraphrase under multi-reference annotation.
$\dagger$~\emph{purported culprit} is a valid paraphrase of \emph{alleged culprit}; counted wrong due to single-reference annotation.
$\ddagger$~\emph{play ticket} is a valid paraphrase of \emph{toy ticket}; counted wrong because the annotated reference is \emph{pass lookalike}.}
\label{tab:error}
\end{table*}

\section{Backbone Grid on DBpedia}
\label{app:backbone}

Table~\ref{tab:backbone} reports R@10 and MRR for all five backbones on DBpedia (245 test queries), both frozen and after 300 in-domain fine-tuning steps.

\begin{table}[h]
\centering
\small
\setlength{\tabcolsep}{4pt}
\begin{tabular}{lrr rr}
\toprule
 & \multicolumn{2}{c}{Frozen} & \multicolumn{2}{c}{Fine-tuned} \\
\cmidrule(lr){2-3}\cmidrule(lr){4-5}
Backbone & R@10 & MRR & R@10 & MRR \\
\midrule
\texttt{all-mpnet-base-v2}   & 0.608 & 0.349 & \textbf{0.845} & \textbf{0.550} \\
\texttt{jina-v5-nano}        & \textbf{0.784} & \textbf{0.431} & 0.829 & 0.468 \\
\texttt{e5-base-v2}          & 0.551 & 0.299 & 0.816 & 0.519 \\
\texttt{para-mpnet-base-v2}  & 0.624 & 0.307 & 0.816 & 0.508 \\
\texttt{para-MiniLM-L6}      & 0.539 & 0.248 & 0.678 & 0.360 \\
\bottomrule
\end{tabular}
\caption{Backbone grid on DBpedia (245 test queries, 2,849 candidates).
All fine-tuned models use 300 DBpedia-specific steps.
Bold: best per column. Contrastive fine-tuning improves every backbone tested.}
\label{tab:backbone}
\end{table}

\section{Use of AI Writing Assistance}

During the preparation of this work the author(s) used Claude (Anthropic) in order to assist with experiment implementation (coding), writing, editing, and LaTeX preparation of the manuscript. After using this tool, the author(s) reviewed and edited the content as needed and take(s) full responsibility for the content of the publication.

\end{document}